\documentclass[review]{IEEEtran}

\usepackage[utf8]{inputenc} 
\usepackage[T1]{fontenc}    
\usepackage{hyperref}       
\usepackage{url}            
\usepackage{booktabs}       
\usepackage{amsfonts}       
\usepackage{amssymb}        
\usepackage{amsthm}         
\usepackage{amsmath}        
\usepackage{nicefrac}       
\usepackage[english]{babel} 
\usepackage{color}          
\usepackage{graphicx}       
\usepackage{enumitem}
\usepackage{algorithm}
\usepackage{algorithmicx}
\usepackage{algpseudocode}
\usepackage{fixltx2e}
\usepackage{array}
\usepackage{tabu}
\usepackage{multirow}

\newcommand{\R}{\mathbb{R}}

\newcommand{\prob}{\mathbb{P}}
\newcommand{\frmu}{\mu}
\DeclareMathOperator{\frvar}{V_f}

\newcommand{\expect}{\mathbb{E}}
\newcommand{\dzvector}[1]{\mathbf{\boldsymbol{#1}}}
\renewcommand{\vec}[1]{\dzvector{#1}}
\newcommand{\mc}[1]{\mathcal{#1}}
\DeclareMathOperator{\map}{\phi}
\DeclareMathOperator*{\argmin}{arg\,min}
\DeclareMathOperator*{\argmax}{arg\,max}
\DeclareMathOperator{\graphdist}{\delta}
\DeclareMathOperator{\mahal}{d}
\newcommand{\norm}[2]{\left|{#1}\right|_{#2}}
\DeclareMathOperator{\dom}{Dom}
\DeclareMathOperator{\supp}{Supp}

\DeclareMathOperator{\test}{Test}
\newcommand{\pval}{p_{\rm val}}

\newcommand{\mucpm}{$\mu$-CPM}
\newcommand{\ecpm}{$\mc E$-CPM}
\newcommand{\ediv}{E-div}

\theoremstyle{plain}
\newtheorem{lemma}{Lemma}
\newtheorem{proposition}{Proposition}

\definecolor{darkgreen}{RGB}{30,150,60}

\begin{document}

\title{Change Point Methods on a Sequence of Graphs}

\author{Daniele Zambon,~\IEEEmembership{Student Member,~IEEE},
        Cesare Alippi,~\IEEEmembership{Fellow,~IEEE},
        and Lorenzo Livi,~\IEEEmembership{Member,~IEEE}
\thanks{Manuscript received ; revised .}
\thanks{Daniele Zambon is with the Faculty of Informatics, 
Universit\`{a} della Svizzera italiana, 
Lugano, Switzerland 
(e-mail: daniele.zambon@usi.ch).}
\thanks{Cesare Alippi is with the Dept. of Electronics, Information, and Bioengineering, 
Politecnico di Milano, 
Milan, Italy and 
Faculty of Informatics, 
Universit\`{a} della Svizzera italiana, 
Lugano, Switzerland 
(e-mail: cesare.alippi@polimi.it, cesare.alippi@usi.ch)}
\thanks{Lorenzo Livi is with the Department of Computer Science, College of Engineering, Mathematics and Physical Sciences, 
University of Exeter, 
Exeter EX4 4QF, United Kingdom 
(e-mail: l.livi@exeter.ac.uk).}
}

\maketitle

\begin{abstract}
Given a finite sequence of graphs, e.g., coming from technological, biological, and social networks, the paper proposes a methodology to identify possible changes in stationarity in the stochastic process generating the graphs.
In order to cover a large class of applications, we consider the general family of attributed graphs where both topology (number of vertexes and edge configuration) and related attributes are allowed to change also in the stationary case.
Novel Change Point Methods (CPMs) are proposed, that (i) map graphs into a vector domain; (ii) apply a suitable statistical test in the vector space; (iii) detect the change --if any-- according to a confidence level and provide an estimate for its time occurrence.
Two specific multivariate CPMs have been designed: one that detects shifts in the distribution mean, the other addressing generic changes affecting the distribution.
We ground our proposal with theoretical results showing how to relate the inference attained in the numerical vector space to the graph domain, and vice versa. We also show how to extend the methodology for handling multiple change points in the same sequence.
Finally, the proposed CPMs have been validated on real data sets coming from epileptic-seizure detection problems and on labeled data sets for graph classification. Results show the effectiveness of what proposed in relevant application scenarios.
\end{abstract}

\section{Introduction}
\label{sec:intro}

A graph representation for data is appropriate in several fields, including physics, chemistry, neuroscience, and sociology \cite{newman2010networks}, where the phenomena under investigations can be observed as a sequence of measurements whose pairwise relationships are relevant too and thus included in the data representation \cite{li2017fundamental}.
In these application scenarios, the identification of a possible change in the system behavior, a situation associated with anomalies or events to be detected in the sequence, is of particular interest; 
examples of applications that can be cast in this framework are functional brain networks \cite{KHAMBHATI20161170} and power grids \cite{powerlosses1}. Further relevant applications cover data acquired from cyber-physical systems and the Internet of Things \cite{alippi2017not}.

In all above applications the topology and the number of vertexes at different time steps may vary, and attributes can be associated with both vertexes and edges; moreover, attributes are not limited to numeric ones and may include categorical data, strings and even a mix of multiple types. In order to cover all these scenarios, we formalize graphs as objects of a graph domain $\mc G$ belonging to the family of graph alignment spaces (GASs) \cite{jain2016geometry}.
GASs provide a metric structure which is also capable to deal with isomorphic graphs, i.e., they account for the case where a one-to-one correspondence among vertexes of different graphs is unavailable or missing.

An illustrative example of graph sequence is provided by Figure \ref{fig:illustrative-ex}, where the graphs resemble the `A' character in the first part of the sequence, until a certain time step $t^*$ after which the graph changes to a new one, resembling the `E' character. 
%
%
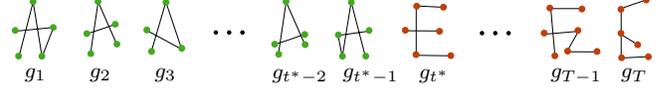
\begin{figure}
\def\svgwidth{\columnwidth}
{\footnotesize

\begingroup%
  \makeatletter%
  \providecommand\color[2][]{%
    \errmessage{(Inkscape) Color is used for the text in Inkscape, but the package 'color.sty' is not loaded}%
    \renewcommand\color[2][]{}%
  }%
  \providecommand\transparent[1]{%
    \errmessage{(Inkscape) Transparency is used (non-zero) for the text in Inkscape, but the package 'transparent.sty' is not loaded}%
    \renewcommand\transparent[1]{}%
  }%
  \providecommand\rotatebox[2]{#2}%
  \newcommand*\fsize{\dimexpr\f@size pt\relax}%
  \newcommand*\lineheight[1]{\fontsize{\fsize}{#1\fsize}\selectfont}%
  \ifx\svgwidth\undefined%
    \setlength{\unitlength}{374.17322835bp}%
    \ifx\svgscale\undefined%
      \relax%
    \else%
      \setlength{\unitlength}{\unitlength * \real{\svgscale}}%
    \fi%
  \else%
    \setlength{\unitlength}{\svgwidth}%
  \fi%
  \global\let\svgwidth\undefined%
  \global\let\svgscale\undefined%
  \makeatother%
  \begin{picture}(1,0.12878788)%
    \lineheight{1}%
    \setlength\tabcolsep{0pt}%
    \put(0,0){\includegraphics[width=\unitlength,page=1]{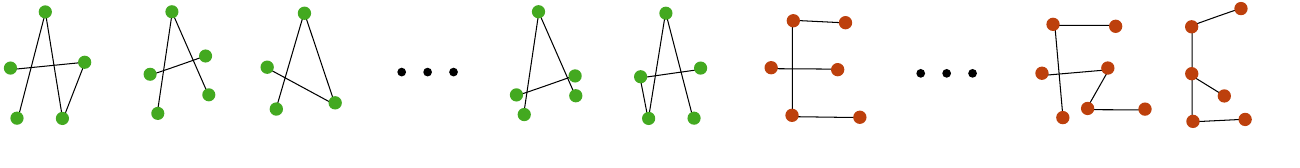}}%
    \put(0.11972644,0.00535786){\color[rgb]{0,0,0}\makebox(0,0)[lt]{\lineheight{0}\smash{\begin{tabular}[t]{l}$g_2$\end{tabular}}}}%
    \put(0.21726725,0.00560842){\color[rgb]{0,0,0}\makebox(0,0)[lt]{\lineheight{0}\smash{\begin{tabular}[t]{l}$g_3$\end{tabular}}}}%
    \put(0.81165359,0.00529966){\color[rgb]{0,0,0}\makebox(0,0)[lt]{\lineheight{0}\smash{\begin{tabular}[t]{l}$g_{T-1}$\end{tabular}}}}%
    \put(0.91758082,0.00565383){\color[rgb]{0,0,0}\makebox(0,0)[lt]{\lineheight{0}\smash{\begin{tabular}[t]{l}$g_{T}$\end{tabular}}}}%
    \put(0.61416439,0.00494523){\color[rgb]{0,0,0}\makebox(0,0)[lt]{\lineheight{0}\smash{\begin{tabular}[t]{l}$g_{t^*}$\end{tabular}}}}%
    \put(0.49981181,0.00512245){\color[rgb]{0,0,0}\makebox(0,0)[lt]{\lineheight{0}\smash{\begin{tabular}[t]{l}$g_{t^*-1}$\end{tabular}}}}%
    \put(0.02212541,0.00535786){\color[rgb]{0,0,0}\makebox(0,0)[lt]{\lineheight{0}\smash{\begin{tabular}[t]{l}$g_1$\end{tabular}}}}%
    \put(0.39369203,0.0048719){\color[rgb]{0,0,0}\makebox(0,0)[lt]{\lineheight{0}\smash{\begin{tabular}[t]{l}$g_{t^*-2}$\end{tabular}}}}%
  \end{picture}%
\endgroup%

\caption{A sequence of $T$ geometric graphs. At each time step $t=1,\dots, T$ a different graph is produced as a measurement of an observed system. Graphs for $t< t^*$ are representative of the normal operating conditions of the observed system (associated with character `A'), whereas for $t\ge t^*$ they represent a change in the system's behavior (associated with character `E'). Notice that observed graphs are not assumed to have a constant structure in either regimes.}
}
\label{fig:illustrative-ex}
\end{figure}

\subsection{Problem formulation}

Consider an unknown stochastic process $\mc P$ that generates a finite sequence of attributed graphs, $\vec g(1,T)=\{g_1, ... , g_T\}$, $g_t\in \mc G$, where $\mc G$ is a GAS.

Each graph of the sequence is interpreted as the realization of a graph-valued random variable.
Under the stationarity hypothesis for process $\mc P$, the $T$ graphs $g_t$, $t=1,2,\dots,T$ are independent and identically distributed%
\footnote{
    Two graphs $g_i,g_j$ are identically distributed if $\prob(g_i\in A)=\prob(g_j\in A)$ for any set $A \subseteq \mc G$; they are independent if $\prob(\{g_i\in A_1\}\cap \{g_j\in A_2\})=\prob(g_i\in A_1)\prob(g_j\in A_2)$, for any pair $A_1,A_2\subseteq\mc G$.
  }
(i.i.d.) according to an unknown distribution $Q_0$ \cite{zambon2017concept}.
Even when the distribution $Q_0$ is stationary, the graphs of the sequence are allowed to vary (in terms of both structure and vertex/edge attributes) from a time step to another; this is, in fact, the random component associated with $Q_0$, which can be seen as the generator of the ``normal operating conditions'' of the monitored system.
Conversely, we say that process $\mc P$ undergoes a change in stationarity if it exists a time $t^*$ such that
\begin{equation}
\label{eq:change-in-stationarity}
\begin{cases}
g_t \sim Q_0 & t  <   t^*\\
g_t \sim Q_1 & t \geq t^*,
\end{cases}
\end{equation}
with $Q_1$ being a graph distribution different from $Q_0$; time $t^*$ is said to be the change point. 
The type of change is abrupt when $\mc P$ commutes from $Q_0$ to $Q_1$ in a single time-step.

As finite sequence $\vec g(1,T)$ is given, to address the detection of a single change in stationarity, we propose to adopt a Change Point Method (CPM) \cite{cabrieto2018testing, hawkins2003changepoint, fan2015identifying, harchaoui2013kernel}, which relies on a series of two-sample statistical tests
$$[s,\pval]=\test(\vec g(1,t-1),\vec g(t,T))$$ 
applied to the $T-1$ pairs of subsets $\vec g(1,t-1)=\{g_1,\dots,g_{t-1}\}$ and $\vec g(t,T)=\{g_{t},\dots,g_{T}\}$, $t=2,\dots,T$, and returns a statistic $s$ together with an associated $p$-value $\pval$.
If at least a test yields a $p$-value lower than a significance level $\alpha$, then a change in stationarity is detected in the sequence $\vec g(1,T)$ and the estimated change point is the one with the lowest $p$-value.
It is important to notice that the CPM framework can be implemented, in principle, by using any two-sample statistical test designed to assess differences between distribution functions.

Not rarely, the driving process $\mc P$ undergoes multiple abrupt changes in the same finite sequence. In these cases, the proposed methodology can be extended by following the E-divisive technique \cite{matteson2014nonparametric}. Whenever the stationarity hypothesis is met for each sub-sequence in between the change points, the methodology addresses the estimation of both the number of change points and their location in the sequence.

\subsection{Contribution and paper organization}

The novelty of our contribution can be summarized as follows:
\begin{itemize}
    \item 
    A methodology to perform change-point analysis on a sequence of attributed graphs by relying on graph embedding. We propose to map each observed graph belonging to a GAS onto a $d$-dimensional point, $d\geq1$, in some vector space where multivariate nonparametric%
	\footnote{Throughout the present paper, the term nonparametric test indicate a test which does not assume a predefined model distribution.}
	statistical hypothesis testing can be applied.

    \item 
    Theoretical developments that allow to control the confidence level of the inference in the graph domain and the vector embedding space. 
    In fact, in Proposition~\ref{prop:fabs-sg-se} we prove that the statistical confidence of the test attained in the embedding space is related to the confidence level that a change has taken place in the graph domain. 
    Our theoretical results show how this relation, in principle, holds true for any graph embedding method.
    
    \item
    We propose two different CPM tests for graphs.
	The first one addresses the common type of change that involves a shift in the mean of the graph distribution; this nontrivially extends the CDT in \cite{zambon2017concept} to the off-line case, and significantly improves the theoretical framework delineated there by removing the strong bi-Lipschitz assumption for the embedding map.
    The second test aims at identifying any kind of change in distribution (Proposition \ref{prop:energy-metric-GAS}) by relying on the energy distance between probability distributions introduced by Szekely \textit{et al.} \cite{szekely2013energy}, which also ensures the consistency of the test.   
    \item 
    We propose an extension of the E-divisive approach \cite{matteson2014nonparametric} in order to identify multiple change points in a graph sequence; the approach is able to estimate both the number of change points and their position in the sequence.

\end{itemize}

The remainder of the paper is organized as follows.
Section \ref{sec:related_works} describes related works.
Preliminary definitions and assumptions are discussed in Section \ref{sec:def-ass}.
Section \ref{sec:methodology} introduces the proposed methodology for performing CPMs on graph sequences.
Section \ref{sec:theory} presents the theoretical results related to the proposed CPMs.
Section \ref{sec:multi-change} shows how to extend the proposed methodology for the identification of multiple changes; related theoretical results are presented in the same section.
To demonstrate the practical usefulness of what proposed, in Section \ref{sec:experiments} we perform simulations on both synthetic data and several real-world data sets of graphs. In particular, we take into account also a relevant real application scenario involving the detection of the onset of epileptic seizures in functional brain connectivity networks.
Section \ref{sec:conclusion} concludes the paper and provides pointers to future research.
Finally, proofs of all theoretical results are provided in Appendix~\ref{sec:proofs}.

\section{Related works}
\label{sec:related_works}

In the literature, CPMs have been initially applied to scalar, normally distributed sequences to monitor shifts in the mean \cite{hawkins2003changepoint} or variance \cite{hawkins2005change}. Extensions have been introduced for nonparametric inference \cite{hawkins2010nonparametric, ross2012two}, multivariate data \cite{zamba2006multivariate,doi:10.1080/02664763.2013.800471,shi2017consistent}, kernel-based inference \cite{li2015m,NIPS2008_3556,NIPS2012_4727}. 
The design of CPMs for sequences of graphs, instead, is still a significantly underdeveloped research area.
The problem of detecting multiple changes, instead, is more challenging \cite{jandhyala2013inference}. One reason lies in the fact that the number of change points is often unknown; furthermore, the identification of one change point may rely on the the identification of the others. To address this problem, it is possible to optimize an objective function for identifying the location of the change and consider a penalty term that takes into account the number of change points \cite{7938741}.
A second direction in the literature aims at tackling the problem in a incremental way, by recursively splitting the original sequence in two parts \cite{matteson2014nonparametric}.

Considering graph sequences, we report the recent contribution in \cite{barnett2016change}, where the authors provide a method to monitor functional magnetic resonance recordings to identify changes. The recordings are modeled by correlation networks and a CPM is applied to detect changes in stationarity.
The technique there proposed, however, is designed for graphs of fixed size with numerical weights associated with edges.
Few other works address changes in sequences of graphs/networks \cite{peel2015detecting,wilson2016modeling}, but none of them operates on the very generic family of attributed graphs considered here.

Some of the already proposed CPMs can be applied to more general input spaces and hence are constrained to operate in vector spaces.
For example, \cite{matteson2014nonparametric, li2015m} rely on a distance or a kernel function, which makes them in principle applicable to graphs. 
However, the theoretical results related to testing for any distribution change in the graph sequence applies when the graph distance is metric of strong negative type \cite{lyons2013distance} and the graph kernel is universal \cite{smola2007hilbert}.

Finally, we note that another relevant type of statistical test is known as Change Detection Test (CDT) \cite{basseville1993detection}. A CDT acts differently from a CPM as it is designed for sequential monitoring.

\section{Background and assumptions}
\label{sec:def-ass}

\subsection{Attributed graphs and the graph alignment space}
\label{sec:def-gas}

An attributed graph $g$ is defined as a triple $(V,E,a)$, where $V$ is a set of vertexes, $E\subseteq V\times V$ is a set of edges and $a:V\cup E\rightarrow \mc A$ is a labeling function that associates to each vertex and edge an attribute taken from a predefined set $\mc A$. 
We denote with $\mc G[\mc A, N_{\max}]$, or simply $\mc G$, the set of all graphs with attributes in $\mc A$ and with at most $N_{\max}<\infty$ vertexes. Notice that $N_{\max}$ can be taken arbitrarily large, but needs to be finite.
Simple examples of attribute sets are $\mc A = [-1,1]$ as in the case of correlation graphs and $\mc A=\R_+$ in the case of transport networks. Set $\mc A$ can be more complex and include vectors, categorical data and strings. Moreover, even combination of multiple types are possible, like in the case of chemical compounds where, e.g., $\mc A$ can contain all chemical elements and possible numbers of valence electrons involved in a bond.

In general, sets $\mc G$ and $\mc A$ are not vector spaces. Hence, operations between graphs, such as computing distances between pairs of graphs $g_i,g_j\in\mc G$, are not trivial \cite{gm_survey,emmert2016fifty}.
In this work, we consider a graph alignment metric (GAM), $\graphdist:\mc G\times \mc G\rightarrow \R_+$ \cite{jain2016geometry} to evaluate the distance between two graphs. A GAM is defined by means of an attribute kernel $k_a:\mc A\times \mc A \rightarrow \R$, which assesses the similarity between attributes as an inner product in an implicit Hilbert space%
    \footnote{Onto which the kernel trick is applied.} 
$\mc H$, and a partial function%
    \footnote{While a function $f: A\rightarrow B$ associates an element $b\in B$ to \emph{every} element $a\in A$, a partial function $pf:A\rightarrow B$ is not necessarily defined for every element $a\in A$, hence is a (proper) function only from a subset $A'\subseteq A$ to set $B$.} 
$\pi:V_i\rightarrow V_j$ called alignment, which associates the vertexes $V_i$ of $g_i$ with vertexes $V_j$ of $g_j$.
A GAM is hence defined as:
$$
\graphdist(g_i,g_j):= \left[ \kappa_{\pi^*}(g_i,g_i)+\kappa_{\pi^*}(g_j,g_j)-2\kappa_{\pi^*}(g_i,g_j)\right]^{1/2},
$$
where
$$
\begin{aligned}
\kappa_{\pi}(g_i,g_j) = 
\sum_{v,v'\in\dom(\pi)} k_a(a((v,v')), a( (\pi(v),\pi(v')) ) ),
\end{aligned}
$$
$\dom(\pi)$ indicates the domain of $\pi$, and
$$
\pi^* =\argmax_{\pi \in \{V_i\rightarrow V_j\}} k_\pi(g_i,g_j),
$$
is the optimal alignment. In order to simplify the notation, we assume no self loops, $(v,v)\not \in E$, and that $a((v,v))=a(v)$ for any vertex $v$. 
The set $\mc G$ equipped with $\graphdist(\cdot,\cdot)$ is called a \emph{graph alignment space} (GAS).
Under the mild assumptions that the attribute kernel is positive semidefinite, $k_a(x,y)\geq 0$ for any $x,y\in\mc A$ and $\mc H = \R^n$ for some $n\in\mathbb{N}$, $(\mc G, \graphdist)$ can be shown to be a metric space \cite[Theo. 4.7,5.2]{jain2016geometry}. This is a sufficient condition and corresponds to Assumption (A1) made in Section~\ref{sec:assumptions}.

We note that several common graph spaces are GASs, e.g., the spaces of weighted graphs equipped with Frobenius norm as distance, and the numeric vector-attributed graph spaces whose distance is based on Euclidean attribute kernels.

\subsection{Graph mean and variation}
\label{sec:def-graph-mean}

In metric --not vector-- spaces the notions of mean and variance have to be adapted, due to the possibly ill-defined operation of summation. This issue can be addressed by considering the formulation given by Fr\'echet \cite{frechet1948elements}, and also adopted in \cite{jiang2001median,jain2016statistical}.
Given a metric GAS $(\mc G,\graphdist)$, consider a graph-valued random variable $g$ taking values in $\mc G$ and distributed according to probability function $Q$. Assume then a sample $\vec g(1,n) =\{g_1,\dots,g_n\}$, with $g_i$, $i=1,\dots,n$ being independent realizations of $g$.

The concepts of graph mean and graph variation (extension of the variance concept) are formalized by the Fr\'echet function $\mc F_Q:\mc G\rightarrow \R_+$ defined for any $g_0\in\mc G$ as
$$
\mc F_Q(g_0) = \int_{\mc G} \graphdist(g_0,g)^2 dQ(g) = \expect_{g\sim Q}[\graphdist(g,g_0)^2].
$$
Function $\mc F_Q(\cdot)$ is positive and, if $\mc F_Q(g_0)$ is finite at some graph $g_0$, then it is finite for any other graph.
Function $\mc F_Q(\cdot)$ is sometimes termed Fr\'echet population function to be distinguished by its
empirical counterpart: the Fr\'echet sample function $\mc F_{\vec g(1,n)}(\cdot)$, which is computed over sample $\vec g(1,n)$,  
$$
\mc F_{\vec g(1,n)}(g_0) = \frac{1}{n}\sum_{i=1}^n \graphdist(g_0,g_i)^2.
$$
The Fr\'echet sample variation is defined as the infimum
$$
\frvar[\vec g(1,n)] =\inf_{g_0\in\mc G}\mc F_{\vec g(1,n)}(g_0).
$$
Such infimum is attained in a finite set of graphs $\frmu_{\vec g(1,n)}\subseteq \mc G$ \cite[Prop. 3.2]{jain2016statistical}, called sample Fr\'echet mean graphs.
Accordingly, the Fr\'echet (population) variation is defined as
$$
\frvar[Q] = \inf_{g_0\in\mc G}\mc F_{Q}(g_0).
$$
A graph attaining the infimum is called Fr\'echet mean graph, and exists whenever $\mc G$ is a complete metric space \cite[Theorem 3.3]{bhattacharya2012nonparametric}; accordingly, $\frmu_{Q}$ denotes the set of all Fr\'echet mean graphs.
Notice that, whenever a Euclidean space $(\mc X,\norm{\cdot}{2})$ is taken into account, the infimum of the Fr\'echet population and sample functions corresponds to the classical expected value $\expect_{x\sim P}[x]=\int_{\mc X}x'\,dP(x')$ and the arithmetic mean $\bar x=n^{-1}\sum_i x_i$, respectively, where $P$ is a probability function on $\mc X$ and $\{x_i\}$ are drawn i.i.d.\ from $P$.

Verifying the uniqueness of the Fr\'echet mean graph in both population and sample cases is more involving; here we limit to a brief discussion and refer to \cite{jain2016statistical} for details. 
A sufficient condition requires that the support of $Q$ is bounded in a ball [Assumption (A2), Sec.~\ref{sec:assumptions}]. Such a ball has to be centered on a graph $g^*$, and has a radius proportional to the degree of asymmetry, $\chi(g^*)$, defined as 
$$
\chi(g):= \sqrt{2}\left[k_{\rm id}(g,g) - \min_{\pi \in \{V\rightarrow V\} } k_\pi(g,g)\right]^{1/2},
$$
with ${\rm id}$ being the identity alignment so that ${\rm id}(v)=v$, $\forall v\in V$.
Graphs with a non-null degree of asymmetry --namely, asymmetric or ordinary graphs-- are spread over the entire graph space $\mc G$ \cite[Cor. 4.19]{jain2016statistical}, and their degree depends on the particular location.
Here, we consider a ball, however, this can be extended to a cone surrounding that ball \cite{jain2016statistical}.
As final remark, we comment that each graph in $g\in\mc G$ can be represented by means of a matrix $M\in \mc A^{N_{\max}\times N_{\max}}$ and $M_{\mc H}\in \mc H^{N_{\max}\times N_{\max}}$, relying on the kernel embedding. It then follows that the Fr\'echet mean exists unique in $\mc H^{N_{\max}\times N_{\max}}$.

In the rest of the paper, when the mean is assumed to exist and be unique, with little abuse of notation we denote $\frmu_{Q}$ as a graph instead of a singleton set.

\subsection{Assumptions}
\label{sec:assumptions}

The current section presents and comments on the assumptions we make throughout the paper. 
As mentioned in the Introduction, here we consider graphs belonging to a graph alignment space $(\mc G, \graphdist)$.
In order to ensure that $\mc G$ has a metric structure, we assume that:
\begin{itemize}
  \item[(A1)] 
  GAM $\graphdist(\cdot,\cdot)$ is built on an positive semi-definite attribute kernel $k_a(x,y)\geq 0$ for any $x,y\in\mc A$ and $\mc H = \R^n$ for some $n\in\mathbb{N}$.
\end{itemize}
Assumption (A1) is mild, yet grants space $(\mc G, \graphdist)$ to be metric, as mentioned in Section~\ref{sec:def-graph-mean}.
The second assumption we make concerns the probability distribution over the graph domain. In particular, we bound the support of the distributions $\{Q_i\}_{i=0}^k$ as follows: 
\begin{itemize}
  \item[(A2)] 
  The Fr\'echet variation $\frvar[Q_i]$ is finite, $i=0,\dots,k$, and there exists a sufficiently asymmetric graph $g^*$ so that $\cup_{i=0}^k\supp(Q_i)$ is bounded by a ball centered in $g^*$ and with radius proportional to $\chi(g^*)$, as requested in \cite{jain2016statistical}.
\end{itemize}
Assumption (A2) grants that the Fr\'echet mean exists unique (see Section~\ref{sec:def-graph-mean}), hence making the mathematics more amenable. 
At the same time, this hypothesis enables Theorem 4.23 in \cite{jain2016geometry} and, accordingly, the ball of graphs is proven isometric to an Euclidean space.
We comment that, as mentioned in Section~\ref{sec:def-graph-mean}, a given GAS is entirely covered by such balls and, moreover, this assumption can be relaxed.

\section{CPMs on a graph sequence}
\label{sec:methodology}

Under the stationarity hypothesis for process $\mc P$, sequence $\vec g(1,T)$ is composed of i.i.d.\ graphs distributed according to the stationary probability function $Q_0$.
Following \eqref{eq:change-in-stationarity}, the statistical hypothesis test for detecting a single change in stationarity can be formulated as
\begin{equation}
\label{eq:cpm-hp-test}
\begin{array}{ll}
H_0:& g_{t} \sim Q_0 ~ \forall t\in\{0,\dots,T\}\\
H_1:& \exists ! \, t^*\in\{2, \dots, T\} ~ \text{s.t.} 
\begin{cases}
  g_t \sim Q_0,& t<t^* \\
  g_t \sim Q_1\ne Q_0,& t\geq t^* 
\end{cases}
\end{array}
\end{equation}

We consider a map $\map:\mc G\rightarrow \R^d$ between the graph domain and the Euclidean space, which associates graph $g\in\mc G$ with point $x=\map(g)\in\R^d, d\geq1$. 
The methodology proposed here, and the related results shown in Section~\ref{sec:theory}, are valid for any embedding function $\map(\cdot)$ that the user chooses, thus making the proposed methodology very general. 
However, in practical applications of our methodology, the distortion introduced by $\map(\cdot)$ plays a relevant role that should be taken into account regardless of the validity of the theoretical results.

By applying the mapping to each graph of the sequence $\vec g(1,T)$, we generate a new transformed sequence $\vec x(1,T)=\{x_1,\dots,x_T\}$ of vectors $x_t=\map(g_t)$, $t=1,\dots,T$.
A multivariate CPM test can then be applied on $\vec x(1,T)$.
In CPMs, one performs multiple two-sample tests.
In particular, for each time index $t=2,\dots,T$, a statistic
$$[\,s_e(t) ,\, \pval(t)\,] =\test(\vec x(1,t-1), \vec x(t,T))$$
is computed on sequences $\vec x(1,t-1), \vec x(t,T)$. 
Notice that, in order to improve readability, and whenever it is clear from the context, we may write $\vec g, \vec x$ instead of $\vec g(1,T), \vec x(1,T)$, and $s_e(t),\pval(t)$ instead of $s_e(t;\vec x(1,T)),\pval(t;\vec x(1,T))$.

Statistical test $\test(\cdot,\cdot)$ depends on the detection problem at hand; for instance, one could design specific tests to identify a change in the mean, or in the variance. Two relevant examples, one addressing changes in the distribution mean, and the other generic changes in the distribution, are discussed later in Sections~\ref{sec:clt} and \ref{sec:energy}, respectively, and Figure \ref{fig:cpm-example} provides a visual description of their application.
The following pseudo-code outlines the CPM for a generic two-sample test.
\begin{algorithmic}[1]
\renewcommand{\algorithmicrequire}{\textbf{Input:}}
\renewcommand{\algorithmicensure}{\textbf{Output:}}
\Require A sequence of observed 
vectors $\vec x(1,T)$; a significance level $\alpha_e$.
\Ensure Whether a change has been detected or not and the estimated change point $\hat t$.
\ForAll{$t=2,\dots,T$} \label{line:for-loop}
  \State Compute $p$-value $\pval(t)$ of $\test(\vec x(1,t-1), \vec x(t,T))$;
\EndFor
\State $\hat t \leftarrow \argmin_t \pval(t)$; \label{line:t-hat}
\If {$\pval\left(\hat t\right) < \alpha_e$}\label{line:test}
  \State Null hypothesis $H_0$ is rejected;
  \State \Return Change detected at time $\hat t$.
\Else
  \State Null hypothesis $H_0$ is not rejected;
  \State \Return No change detected.
\EndIf
\end{algorithmic}

The for-loop in Line~\ref{line:for-loop} explores all possible subdivision of the sequence. Line~\ref{line:t-hat} estimates the candidate change point $\hat t$, and considers graph $g_{\hat t}$ to be the first one drawn by $Q_1$.
Line~\ref{line:test} checks the actual presence of a change; in most cases, the rejection criterion can also be implemented by monitoring the statistic $s_e(\hat t)$, instead of $\pval(\hat t)$, e.g.,
\begin{equation}
\label{eq:cpm-test-rule-emb}
\text{if } s_e(\hat t)>\gamma_e(\hat t) \ \Rightarrow \ \text{reject } H_0 \text{ at significance level } \alpha_e,
\end{equation}
provided that $\gamma_e(\hat t)$ is the quantile of order $1-\alpha_e$ associated with $s_e(\hat t)$.
The significance level 
$$
\alpha_e = \prob(\text{reject } H_0 | H_0)
$$
coincides with the tolerated rate of detected changes (false positive rate) when the null hypothesis $H_0$ holds true: the proposed methodology allows the designer to define the significance level $\alpha_e$ according to the application at hand.
Conversely, the rate of unrecognized changes (false negative rate),
$$
\beta_e = \prob( \text{do not reject } H_0 | H_1)
$$
defines the power (i.e., $1-\beta_e$) of the test. Parameter $\beta_e$ is characteristic of the adopted test $\test(\cdot,\cdot)$ and depends on the family of possible distributions $Q_1$. While $\alpha_e$ can be obtained in non-parametric tests, the value of $\beta_e$ is often unavailable.
Given a significance level $\alpha_e$, the designer can improve the power $1-\beta_e$ by selecting a suitable test;
in fact, the identification of a particular type of change, e.g., a change in the distribution variance, is better addressed when a test specifically designed for that change is adopted.

If a change is detected in sequence $\vec x(1,T)$ with significance level $\alpha_e$, then we say that a change has taken place also in the graph sequence $\vec g(1,T)$ at time $\hat t$. However, the significance level $\alpha_g$ of the inference in the graph domain differs, a priori, from $\alpha_e$ in $\R^d$. Proposition~\ref{prop:fabs-sg-se} shows how the significance levels are related and, accordingly, how a change in the embedding space implies a change in the graph domain and vice-versa.

We point out that some two-sample tests require a minimum sample size, e.g., the Welch's t-test. In those cases, one can add a margin $m>1$ and apply the procedure for $t=m+1,T-m+1$. Considering a margin $m$ is useful also to reduce the false negative rate $\beta_e$; in fact, $\beta_e$ often approaches $1$ when the test \eqref{eq:cpm-test-rule-emb} is applied with $\hat t$ very close to $2$ or $T$, i.e., when the sample size of one of the two subsets is small.

\section{Theoretical results}
\label{sec:theory}

Inferring whether a change in stationarity occurred or not in a sequence of attributed graphs, $\vec g(1,T)$, is a difficult problem. As we do not make assumptions about embedding map $\map(\cdot)$, the resulting sequence $\vec x(1,T)$ does not necessarily encode the same statistical properties of $\vec g(1,T)$.
Nonetheless, here we prove some general results connecting changes in stationarity occurring in the graph sequence $\vec g(1,T)$ with those detected in the embedded sequence $\vec x(1,T)$, and vice-versa.
In order to improve readability, technical details of the various proofs are delivered in Section~\ref{sec:proofs}.

The core of our argument is that, if statistic $s_e(t)=\test_e(\vec x(1,t-1), \vec x(t,T))$ is related to the chosen statistic $s_g(t)=\test_g(\vec g(1,t-1), \vec g(t,T))$ defined in the graph domain, then also their distributions must be related.
By proving this, we can claim that a change occurring in one space can be detected in the other space as well, possibly with different confidence levels.
We mention that, throughout the paper, subscripted `e' and `g' denote quantities associated with the embedding space and the graph domain, respectively.
The following Proposition~\ref{prop:fabs-sg-se} shows how to relate statistics $s_g(t)$ and $s_e(t)$ in probabilistic terms.
In particular, when decision rule \eqref{eq:cpm-test-rule-emb} is applied to $\vec x(1,T)$ with $s_e(t)$ at significance level $\alpha_e$, a decision rule of the type
\begin{equation}
\label{eq:cpm-test-rule-graphs}
\text{if } s_g(t)>\gamma_g(t) \ \Rightarrow \ \text{reject } H_0 \text{ at significance level } \alpha_g
\end{equation}
holds in the graph domain by considering statistic $s_g(t)$.
The significance level $\alpha_g$ for the test on graphs can be bound by means of two significance levels $\alpha_e'$ and $\alpha_e''$ related to the multivariate test in the embedding space.
Significance levels $\alpha_e'$ and $\alpha_e''$ are associated with two threshold $\gamma_e',\gamma_e''$, respectively, that allow the user to perform statistical inference in the embedding space, while controlling the significance level of the corresponding inference in graph domain.
\begin{proposition}
\label{prop:fabs-sg-se}
Consider a sequence $\vec g=\{g_1,\dots,g_T\}$ of graph-valued random variables that are i.i.d.\ according to probability function $Q_0=Q_1$, and assume that (A1), (A2) hold true. Let us define a sequence $\vec x=\map(\vec g):=\{\map(g_1),\dots,\map(g_T)\}$ of random vectors obtained from $\vec g$ through map $\map(\cdot)$.
Let $\Psi_g(\cdot)$ and $\Psi_e(\cdot)$ be the cumulative density functions of statistics $s_g(t; \vec g)$ and $s_e(t; \vec x)=s_e(t; \map(\vec g))$, respectively.
Chosen constants%
\footnote{Constants $\lambda$ and $q$ depend on distribution $Q_0$, but not from $\vec g$.}
$\lambda>0$ and $q\in(0, 1]$ satisfying
\begin{equation}
\label{eq:fabs-sg-se}
\quad\prob_{\vec g\sim Q_0^T}(|s_g(t; \vec g)-s_e(t; \map(\vec g))|\leq \lambda) \geq q,
\end{equation}
then, for any real value $\gamma$, we have that
\begin{equation}
\label{eq:cdf-bound}
q\Psi_e\left(\gamma-\lambda\right)
\leq 
\Psi_g(\gamma) 
\leq 
q^{-1}\Psi_e\left(\gamma+\lambda\right).
\end{equation}
\end{proposition}
With Proposition \ref{prop:fabs-sg-se}, the significance levels, and respective thresholds, can be identified. In fact, by evaluating the bounds in~\eqref{eq:cdf-bound} at $\gamma=\gamma_g(t)$ in \eqref{eq:cpm-test-rule-graphs}
\begin{equation*}
q\Psi_e\left(\gamma_g(t)-\lambda\right)
\leq 
1-\alpha_g
\leq 
q^{-1}\Psi_e\left(\gamma_g(t)+\lambda\right),
\end{equation*}
and defining 
$\alpha_e'  = 1- q^{-1}\Psi_e\left(\gamma_g(t)+\lambda\right)$,
$\alpha_e'' = 1-q\Psi_e\left(\gamma_g(t)-\lambda\right)$,
we obtain that
$\alpha_e' \leq \alpha_g \leq \alpha_e''$;
the associated thresholds $\gamma_e', \gamma_e''$ are those for which $\alpha_e'=\Psi(\gamma_e')$ and $\alpha_e''=\Psi(\gamma_e'')$.
We conclude that, with a confidence \emph{at least} $1-\alpha_g$, if $s_e(t)\le\gamma_e'$ then no change has taken place in the graph domain and, conversely, if $s_e(t)\ge\gamma_e''$ then a change has occurred in the graph domain.
Finally, it is worth observing that, when $\gamma_e'<s_e(t)<\gamma_e''$, it is not possible to make a \emph{reliable} decision as a consequence of the severe distortion introduced by the embedding procedure.
If the embedding is isometric, instead, for any $\lambda>0$, \eqref{eq:fabs-sg-se} holds with probability $q=1$ and
\eqref{eq:cdf-bound} reduces to equality $\Psi_e(\gamma)=\Psi_g(\gamma)$ for any $\gamma$.

A similar reasoning can be done in terms of $p$-values. The subsequent Proposition~\ref{prop:p-vals} shows how to bound the $p$-value $p_g$ associated with graph statistic $s_g(t,\vec g^*)$, evaluated on a specific observed sequence $\vec g^*$, with two $p$-values $p_e',p_e''$ concerning the vector statistic $s_e(t, \map(\vec g^*))$.
\begin{proposition}
\label{prop:p-vals}
Let us consider the assumptions made in Proposition~\ref{prop:fabs-sg-se}, and let $\vec g^*=\{g_1^*,\dots,g_T^*\}$ and $\vec x^*=\{\map(g_1^*), \dots, \map(g_T^*)\}$ be realizations of random sequences $\vec g$ and $\vec x$, respectively.
Further, let
\begin{align*}
  p_g&  =\prob_{\vec g\sim Q_0^T}(s_g(t;\vec g)> s_g(t;\vec g^*) )
\\p_e'& =\prob_{\vec g\sim Q_0^T}(s_e(t;\map(\vec g))> s_e(t;\vec x^*)+2\lambda)
\\p_e''&=\prob_{\vec g\sim Q_0^T}(s_e(t;\map(\vec g))> s_e(t;\vec x^*)-2\lambda)
\end{align*}
be the $p$-values associated with $s_g(t;\vec g^*)$, $s_e(t;\vec x^*)\pm 2\lambda$, respectively.
Then, with probability $q$
\begin{equation}
\label{eq:p-vals}
q^{-1}\,p_e' + 1 - q^{-1}
\leq p_g  \leq 
q\,p_e'' + 1 - q.
\end{equation}
\end{proposition}

The following subsections propose two CPM tests based on choices of distances for graphs and vectors that are relevant in common applications, for which (i) terms in \eqref{eq:fabs-sg-se} can be made explicit, and (ii) the distribution of $s_e(t;\vec x)$ can be determined. It follows that Propositions \ref{prop:fabs-sg-se} and \ref{prop:p-vals} hold true and the distribution of $s_g(t;\vec x)$ can be derived from that of $s_e(t;\vec x)$, together with the confidence level $1-\alpha_g$.

\subsection{A CPM test for a shift in the Fr\'echet mean}
\label{sec:clt}

The first CPM test we propose addresses the detection of a shift in the mean of the graph distribution. The derived statistical hypotheses are
\begin{align*} 
H_0&:\frmu_{Q_0}=\frmu_{Q_1} \\ 
H_1&:\frmu_{Q_0}\neq\frmu_{Q_1}. 
\end{align*} 
The adopted statistic for the embedding space is defined as:
$$
s_e(t;\vec x)=T\mahal_M^2(\frmu_{\vec x(1,t-1)},\frmu_{\vec x(t,T)}),
$$
which is based on the squared Mahalanobis distance:
\begin{multline}
\label{eq:mahal-distance}
\mahal_M^2\left(\frmu_{\vec x(1,t-1)},\frmu_{\vec x(t,T)}\right) = \\
\left(\frmu_{\vec x(1,t-1)} - \frmu_{\vec x(t,T)}\right)^\top M^{-1}\left(\frmu_{\vec x(1,t-1)} - \frmu_{\vec x(t,T)}\right).
\end{multline}
$\frmu_{\vec x(1,t-1)}$ and $\frmu_{\vec x(t,T)}$ are the sample means and $M$ is the pooled sampling covariance matrix.
Under the stationarity hypothesis, the distribution of statistic $s_e(t;\vec x)$ can be determined; in fact, by applying the central limit theorem, $s_e(t;\vec x)$ is asymptotically distributed as a $\chi^2(d)$, where $d$ denotes the embedding space dimension.
As the distribution of $s_e(t;\vec x)$ is now available in closed-form, a threshold $\gamma_e$ can be set to control the false positive rate.

Accordingly, we define the graph statistic $s_g(t;\vec g)$ as the squared GAM $\graphdist^2(\cdot,\cdot)$ between the mean graphs $\frmu_{\vec g(1,t-1)} $and $\frmu_{\vec g(t,T)}$,
$$
s_g(t;\vec g) = T\,\graphdist^2\left(\frmu_{\vec g(1,t-1)},\frmu_{\vec g(t,T)}\right).
$$
Let us recall that, as we are considering attributed graphs, possibly with a variable number of vertexes, the mean graph elements are intended according to Fr\'echet, as described in Section~\ref{sec:def-graph-mean}. 
Further, we highlight that $\mahal_M(\frmu_{\vec x(1,t-1)},\frmu_{\vec x(t,T)})$ and $\graphdist(\frmu_{\vec g(1,t-1)},\frmu_{\vec g(t,T)})$ are consistent estimators of 
$\mahal_M(\frmu_{F_0},\frmu_{F_1})$ and $\graphdist(\frmu_{Q_0},\frmu_{Q_1})$, respectively, where $F_i$ is the distribution of random vector $x=\map(g)$, for $g\sim Q_i$ and $i\in\{0, 1\}$. In fact, the ordinary and Fr\'echet sample means are consistent estimators of their population counterparts \cite{jain2016statistical}. 

With the above selection for statistics $s_g(t)$ and $s_e(t)$, 
the claim of Proposition~\ref{prop:fabs-sg-se} can be refined and made more explicit.
This is done in following Lemma~\ref{lemma:fabs-sg-se-clt}, which explicitly provides a $q$ for any $\lambda>0$ (see Eq.~\ref{eq:fabs-sg-se}). 
\begin{lemma}
\label{lemma:fabs-sg-se-clt}
Under the assumptions of Proposition~\ref{prop:fabs-sg-se}, there exists a positive constant $V_1(t)$ depending on distribution $Q$ and time $t$, such that, for any $\lambda>0$
$$
\prob\left(|s_e(t;\vec x)-s_g(t;\vec g)| \leq \lambda \right)\geq 1- \lambda^{-1}V_1(t),
$$
with 
$V_1(t)=\tfrac{2 T^2}{(t-1)(T-t+1)}\left(\lambda_d({M})^{-1} \frvar[F]+\frvar[Q]\right).$   
\end{lemma}
$\lambda_d({M})$ is the smallest eigenvalue of matrix $M$, and $\frvar[\cdot]$ is the Fr\'echet variation, see Section~\ref{sec:def-graph-mean}; see Section~\ref{sec:proof-lemma:fabs-sg-se-clt} for a proof and detailed explanation.
From above lemma, it follows that Proposition~\ref{prop:fabs-sg-se} holds true for any positive value of $\lambda$, with $q=q(\lambda) = 1- \lambda^{-1}V_1(t)$.
We point out that the constant $V_1(t)$ is proportional to the sum of Fr\'echet variation of $Q_0$ and $F_0$ and therefore it can be considered as a measure for the data spread in both graph and embedding spaces.
%

\subsection{A CPM test to assess generic distribution changes}
\label{sec:energy}
The second proposed CPM test allows to identify any type of changes in stationarity affecting the distribution. As such, the hypothesis test can be formalized as $H_0:Q_0=Q_1$ against $H_1:Q_0\neq Q_1$. The multivariate two-sample test adopted in this CPM test is based on the energy statistic $\mc E(\cdot,\cdot)$ \cite{szekely2004testing} and, accordingly, the statistic in the embedding space is 
\begin{equation}
\label{eq:energy-statistic}
s_e(t;\vec x) = \tfrac{(t-1)(T-t+1)}{T} \mc E(\vec x(1,t-1), \vec x(t, T)),
\end{equation}
with 
\begin{multline}
\label{eq:sample-energy-distance}
\mc E(\vec x(1,t-1), \vec x(t, T)) :=  \frac{2\sum_{i=1}^{t-1}\sum_{j=t}^{T} \norm{x_i-x_j}{2}}{(t-1) (T-t+1)} \\
- \frac{\sum_{i,j=1}^{t-1}  \norm{x_i-x_j}{2}}{(t-1)^2}  - \frac{\sum_{i,j=t}^{T} \norm{x_i-x_j}{2}}{(T-t+1)^2}.
\end{multline}
Asymptotically $s_e(t;\vec x)$ follows a weighted sum of $\chi^2(1)$ distributions, provided the variance of $x_i\in\vec x$ is finite and associated $p$-values can be computed via permutation  \cite{szekely2004testing}.
Sz\'ekely and Rizzo \cite{szekely2004testing} showed also that tests based on $s_e(\cdot)$ are consistent when testing equality of distributions $F_0$ and $F_1$ against the $F_0\neq F_1$ hypothesis, implying that the test is able to detect any discrepancy between distributions.
This follows from the fact that statistic $s_e(t;\vec x)$ is the empirical version of the energy distance $E^2(F_0,F_1)$, which is proven to be a metric distance between distributions $F_0$ and $F_1$ with support on Euclidean spaces,
\begin{multline}
\label{eq:energy-distance}
E^2(F_0,F_1):= 
2\expect[\norm{x_0-x_1}{2}]
\\ - \expect[\norm{x_0-x_0'}{2}] - \expect[\norm{x_1-x_1'}{2}],
\end{multline}
with $x_0,x_0'\sim F_0$ and $x_1,x_1'\sim F_1$ independent random vectors.
Such a property can be extended to more general metric spaces \cite{szekely2013energy} by substituting in Equation~\eqref{eq:energy-distance} the associated metric distance.
In particular, as stated in Proposition~\ref{prop:energy-metric-GAS}, this holds for $(\mc G, \graphdist)$ whenever the graph domain is a proper GAS. 
\begin{proposition} 
\label{prop:energy-metric-GAS}
Let us define $\mc D$ as the set of all probability functions on a measurable space over $(\mc G,\graphdist)$, so that $\bigcup_{Q\in\mc D}Q$ fulfills the support condition of (A2). Then, use in $E^2(\cdot,\cdot)$ of \eqref{eq:energy-distance} the metric distance between samples $\graphdist(\cdot,\cdot)$.
It follows that $(\mc D, E)$ is a metric space.
\end{proposition} 
Supported by this fact, we consider as graph statistic
\begin{equation*}
\begin{aligned}
s_g(t;\vec g) =& \tfrac{(t-1)(T-t+1)}{T}\left\lbrace \frac{2\sum_{i=1}^{t-1}\sum_{j=t}^{T} \graphdist(g_i,g_j)}{(t-1) (T-t+1)} \right.\\
&- \left.\frac{\sum_{i,j=1}^{t-1}  \graphdist(g_i,g_j)}{(t-1)^2}  - \frac{\sum_{i,j=t}^{T} \graphdist(g_i,g_j)}{(T-t+1)^2}\right\rbrace.
\end{aligned}
\end{equation*}

Similarly to what we proved for Lemma~\ref{lemma:fabs-sg-se-clt}, Lemma~\ref{lemma:fabs-sg-se-energy} connects statistics $s_e(\cdot)$ and $s_g(\cdot)$ in probability.
This result will then be used in Proposition~\ref{prop:fabs-sg-se} to obtain an explicit relation between confidence levels for the test based on the energy distance. 
\begin{lemma}
\label{lemma:fabs-sg-se-energy}
Under the assumptions of Proposition~\ref{prop:fabs-sg-se}, there exists a positive constant $V_2$ depending on distribution $Q$, such that, for any $\lambda>0$
\begin{equation*}
\prob(|s_e(t;\map(\vec g))-s_g(t; \vec g)|\leq \lambda|H_0) \geq 1-\lambda^{-1}V_2,
\end{equation*}
with 
$V_2=2\,(\expect_{g\sim Q}[\graphdist(g,\frmu_Q)]+\expect_{x\sim F}[\norm{x-\frmu_F}{2}]).$
\end{lemma}
Lemma~\ref{lemma:fabs-sg-se-energy} provides a way to define constant $q$ in terms of $V_2$, i.e., $q=q(\lambda) = 1-\lambda^{-1}V_2$.
In this sense, Lemma~\ref{lemma:fabs-sg-se-energy} is analogous to Lemma~\ref{lemma:fabs-sg-se-clt}; moreover, quantities 
$\expect_{g\sim Q}[\graphdist(g,\frmu_Q)]$ and $\expect_{x\sim F}[\norm{x-\frmu_F}{2}]$ are measures of distribution spread and, accordingly, also Lemma~\ref{lemma:fabs-sg-se-energy} can be interpreted in terms of the data uncertainty.

\section{Identifying multiple change points}
\label{sec:multi-change}

What discussed so far assumes that input sequences contain at most one change point.
Here, we elaborate over the E-divisive approach \cite{matteson2014nonparametric} and design a CPM test able to detect multiple abrupt change points (if present) in a sequence of attributed graphs.
The E-divisive approach relies on a two-sample test $\test(\cdot,\cdot)$ and produces, for a generic sequence $\vec x(a,b), 1\leq a<b\leq T$, a statistic $s_e(t; \vec x(a, b))$ based on the energy distance $\mc E(\cdot,\cdot)$ defined in Equation~\ref{eq:sample-energy-distance}:
\begin{equation}
\label{eq:edivisive-se}
s_e(t; \vec x(a, b)) = \max_{t\le r\le b} 
\mc E(\vec x(a,t-1), \vec x(t, r)).
\end{equation}
As commented below, the maximum over variable $r$ is introduced to take into account the possible presence of multiple change points in the same sequence.

Multiple abrupt change points are detected incrementally.
The algorithm initially takes into account the entire (embedded) sequence $\vec x(1,T)$, and selects time step $\hat t$ so as to maximize the test statistic $s_e(t; \vec x(1, T))$,
\begin{equation}
\label{eq:initial-e-divisive}
\hat t = \argmax_{{2\le t\le T}} \;s_e(t; \vec x(1, T)).
\end{equation}
Time index $\hat t$ is the first discovered change point, provided that the associated $p$-value is lower than a predefined significance level $\alpha_e$. This step looks fairly similar to a typical CPM test for a single change, as the idea is to sweep over all bi-partitions induced by $t=2,\dots,T$.
However, we stress a fundamental difference introduced by the auxiliary variable $r$. In fact, by varying $r$ we can mitigate possible side effects deriving, e.g., from the presence of multiple distributions in $\vec x(t, T)$, making it statistically indistinguishable from $\vec x(1, t-1)$.

In order to describe a generic iteration of the E-divisive technique, let us assume that a set of $k$ different change points, $\{0=\hat t_0 < \hat t_1 < \dots <\hat t_{k-1} < \hat t_k=T\}$, has been already identified (to simplify the notation, endpoints $\hat t_0=0$ and $\hat t_k=T$ are considered change points as well).
The next iteration consists in applying the procedure described above for $i=0,\dots,k-1$, to a sub-sequence $\vec x(\hat t_i+1, \hat t_{i+1})$, obtaining via \eqref{eq:initial-e-divisive} a new candidate change point $\hat t^{(i)}$.
Among these candidate change points, the new change point $\hat t$ is the time index maximizing the associated statistic, i.e., $\hat t=\hat t^{(j)}$, with
\begin{equation*}
j = \argmax_{i\in\{0,\dots,k-1\}} \;s_e\left(\hat t^{(i)}; \vec x\left(\hat t_i, \hat t_{i+1}-1\right)\right).
\end{equation*}
Again, if the $p$-value associated with $\hat t$ is lower than a predefined significance level $\alpha_e$, then $\hat t$ is retained as an additional change point. The procedure is repeated until the outcome of a test is not statistically significant, meaning that all change points present in the sequence have been identified.

The statistic $s_e(t; \vec x(a,b))$ used for the detection in the embedding space can be associated with the graph statistic:
\begin{multline*}
s_g(t;\vec g(a,b)) = \max_{t \le r\le b} \left\{\frac{2\sum_{i=a}^{t-1}\sum_{j=t}^{r} \norm{g_i-g_j}{2}}{(t-a) (r-t+1)} \right.\\
\left.- \frac{\sum_{i,j=a}^{t-1}  \norm{g_i-g_j}{2}}{(t-a)^2} - \frac{\sum_{i,j=t}^{r} \norm{g_i-g_j}{2}}{(r-t+1)^2}\right\}.
\end{multline*}
Similarly to what discussed in Section~\ref{sec:theory}, here we prove the following Lemma~\ref{lemma:fabs-sg-se-edivisive} which demonstrates how to relate the significance levels in the graph and embedding domains.
\begin{lemma}
\label{lemma:fabs-sg-se-edivisive}
Let us consider again the assumptions made in Proposition~\ref{prop:fabs-sg-se}.
Let $\vec g=\{g_a,\dots,g_b\}$, $1\le a < b \le T$, be a sequence of graphs and let $\vec x=\map(\vec g)$ be the associated sequence in the embedding space.
Then, there exists a positive constant $V_3(t)$ that depends on distribution $Q$ and time $t$, such that, for any $\lambda>0$
\begin{equation*}
\prob_{\vec g\sim Q^{b-a}}(|s_e(t;\map(\vec g))-s_g(t; \vec g)|\leq \lambda|H_0) \geq 1-\lambda^{-1}V_3(t),
\end{equation*}
with $V_3(t)=2\big(\tfrac{b-a}{t-a}+\log(b-a+1)\big)      
\big(\expect_{g\sim Q}[\graphdist(g,\frmu_Q)]
+\expect_{x\sim F}[\norm{x-\frmu_F}{2}]\big)$.
\end{lemma}

Lemma~\ref{lemma:fabs-sg-se-edivisive} is used every time a new candidate change point $\hat t^{(j)}$ is found via Eq.~\ref{eq:initial-e-divisive}, and with extrema $a=\hat t_j$ and $b=\hat t_{j+1}-1$.
The bound above takes different values depending on the candidate change point, as it happens also with Lemma~\ref{lemma:fabs-sg-se-energy}. In particular, the ratio $(b-a)/(t-a)$ in $V_3(t)$ can be interpreted as the inverse of the relative location of $t$ in the interval $[a,b]$ under analysis. 
Notice that its value is unbounded when $t$ approaches $a$, as the estimation of the expectation of $s_e(t)$ in Eq.~\ref{eq:edivisive-se} involves computing a maximum value. This issue can be mitigated by considering a margin $m$ so that $t$ is selected in the range $a+m,\dots,b-m$. Moreover, as mentioned in Section~\ref{sec:methodology}, such a margin would also be useful to avoid issues related to the power of the test.

\section{Experiments}
\label{sec:experiments}

\begin{table}
    \centering    
    \caption{Data sets taken into account.}
    \label{tab:exp-data}
    \begin{tabular}{|ccccc|}
        \hline
        
        \hline
        {\textbf{ID}} &  {\textbf{Database}} & {\textbf{Data set}} & {\textbf{Classes}} & \textbf{\# graphs} \\
        \hline
        
        \hline
        \multirow{2}{*}{Del} & \multirow{2}{*}{Delaunay} & \multirow{2}{*}{Delaunay} & 0, 6, 8, 10, 12, & \multirow{2}{*}{(100, ..., 100)}\\
                            &  &  & 14, 16, 18, 20 & \\
        \hline
        Let & IAM & Letter & A, E, F, H, I & (150, ..., 150)\\
        \hline
        \multirow{2}{*}{AIDS} & \multirow{2}{*}{IAM} & \multirow{2}{*}{AIDS} & 0 (inactive), & \multirow{2}{*}{(1600 ,400)}\\
                & &        & 1 (active) &  \\
        \hline
        \multirow{2}{*}{Mut} & \multirow{2}{*}{IAM} & \multirow{2}{*}{Mutagenicity} & 0 (nonmutag.) & \multirow{2}{*}{(1963 ,2401)} \\ 
                   &&              & 1 (mutag.) & \\ 
        \hline
        D1 & Kaggle & Dog1 &  & (418, 178) \\
        D2 & Kaggle & Dog2 & 0 (preictal), & (1148, 172) \\
        D3 & Kaggle & Dog3 & 1 (ictal) & (4760, 480) \\
        D4 & Kaggle & Dog4 &  & (2790, 257) \\
        \hline
        H1 & Kaggle & Human1 &  & (104, 70) \\
        H2 & Kaggle & Human2 &  & (2990, 151) \\
        H3 & Kaggle & Human3 & 0 (preictal), & (714, 327) \\
        H4 & Kaggle & Human4 & 1 (ictal) & (190, 20) \\
        H5 & Kaggle & Human5 &  & (2610, 135) \\
        H6 & Kaggle & Human6 &  & (2772, 225) \\
        H7 & Kaggle & Human7 &  & (3239, 282) \\
        H7 & Kaggle & Human8 &  & (1710, 180) \\
        \hline
    \end{tabular}
\end{table}
\begin{figure}
\centering
\includegraphics[scale=0.4,keepaspectratio=true]{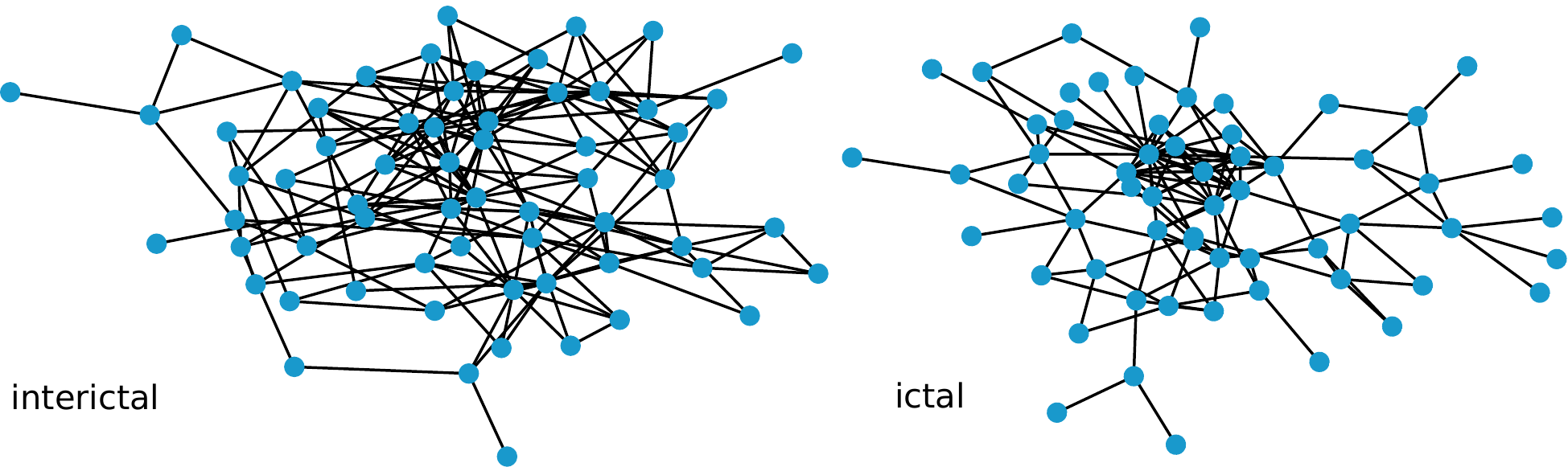}
\vspace{-.2cm}
\caption{Two example graphs extracted from interictal and ictal classes of subject H1, respectively. The graphs are represented by drawing only those edges whose attributes (Pearson correlation) are greater than $0.2$.}
\label{fig:graphs}
\end{figure}

Here, we perform simulations in order to assess the effectiveness of the proposed CPMs.
We will consider both synthetically generated sequences of graphs and real data.
The real-world data come from different application domains, including bio-molecules and electro-electroencephalograms (EEG), thus covering different case studies of practical relevance.

\subsection{Data}

Table \ref{tab:exp-data} summarizes the main characteristics of the data sets taken into account, which are furhter described in the following sections. 

\subsubsection{Delaunay graphs}
As synthetic --controlled-- data, we take into account the Delaunay graphs first introduced in \cite{zambon2017detecting}.
Delaunay graphs are geometric graphs composed by 7 vertexes and 2-dimensional real coordinates as vertex attributes; the topology of the graph is defined by the Delaunay triangulation of the vertexes. Different classes of graphs can be generated by considering different coordinates for the vertexes. For a detailed description of the generation process, we refer the reader to Zambon \textit{et al.}~\cite{zambon2017detecting}.
Changes in stationarity along a sequence of Delaunay graphs are simulated by inducing a transition between different classes of graphs. Delaunay data set constains several classes. In particular, here we will consider a reference class, called ``class 0''. Class 1 contains instance graphs very different from those in class 0. As the class index increases, the graph instances of that class become similar to those of class 0, so that, e.g., distinguishing class 0 from class 8 is easier than distinguishing class 0 from class 12. Accordingly, detecting a change is more difficult in the latter case. In this paper, we consider classes 0, 6, 8, 10, 12, 14, 16, 18, and 20 as reported in Table~\ref{tab:exp-data}.

\subsubsection{IAM database}
We will experiment on three data sets from the IAM graph database \cite{riesen+bunke2008}, namely the Letter, AIDS and Mutagenicity data sets.
Letter data set is composed of handwriting letters represented as graphs. 
There are 15 classes, each of which is associated to a different letter for a total of $15\times 150 = 2250$ graphs.
Graphs are characterized by a variable topology and number of vertexes (from 2 to 9 vertexes); a 2-dimensional vector is associated to each vertex as attribute. AIDS and Mutagenicity data sets contain graph representations of biological molecules.
Both data sets contain two classes of graphs. 
Originally, the Letter data set comes in three different versions; here we consider the data set having the highest variability in order to make the problem more difficult. 
The AIDS data set contains 1600 inactive graphs and 400 graphs representing active molecules.
Mutagenicity data set contains 1963 nonmutagenic molecules and 2401 mutagenic molecules.
In both data sets, the graphs are characteriezd by chemical symbols as vertex attributes (i.e., categorical data) and valence of the chemical links as edge attributes.
The AIDS data set contains graphs with as much as 95 vertexes; Mutagenicity data set contains larger graphs with up to 417 vertexes.

\subsubsection{Detection of epileptic seizures from iEEGs}
The final case study we take into account refers to the problem of detecting epileptic seizures from intracranial electro-encephalogram (iEEG) recordings. Notably, we use the ``Detect seizures in intracranial EEG recordings'' database by UPenn and Mayo Clinic.%
    \footnote{\url{https://www.kaggle.com/c/seizure-detection}} 
The database contains recordings related to different subjects (8 humans and 4 dogs).
The recordings belong to two classes, denoting interictal and ictal segments. The first one refers to recordings denoting normal brain activity, while the second one refers to seizure events.
Each subject data set contains a pre-defined split in training and a test set, but only the training clips are labeled. Accordingly, in order to rely on a ground truth change point, here we will consider the training set of each subject only.

For each subject, the data is available as a sequence of one-second clips with a variable number of channels (from 16 to 72), giving rise to a multivariate stream of iEEGs. 
In order to model (statistical) coupling among the activity recorded in different brain regions, it is common to represent iEEG data as functional connectivity networks \cite{bastos2016tutorial}. Functional connectivity networks are weighted graphs, where (usually) the vertexes correspond to the signals recorded by the electrodes (or channels of the electrodes) and the edge weights represent their coupling strength. Many connectivity measures have been proposed for this purpose: here we consider Pearson correlation computed in the high-gamma band (70-100Hz). We also characterize each vertex with the leading four wavelet coefficients \cite{bastos2016tutorial} computed from the related raw signals by means of the discrete wavelet transform. Figure~\ref{fig:graphs} provides a visual representation of two example graphs associated with different regimes.

\subsection{Experimental setup and implementation details}
\label{sec:setup}
To obtain a stationary sequence of graphs from one of the above-mentioned data sets, we select all graphs of one class and randomly arrange them in sequence.
In order to simulate $k\ge 1$ change points, we generate $k+1$ stationary sequences from different classes and, finally, concatenate them to form a single longer sequence $\vec g$.
Throughout the rest of the paper, we indicate a particular sequence $\vec g$ with the ID of the considered data set (see Table~\ref{tab:exp-data}) and with sub-scripted the list of classes indicating the order by which they occur in the sequence $\vec g$. For example, experiment ``Let\textsubscript{A,E,F}'' refers to the Letter data set and considers a sequence $\vec g(1,450)$, where $\vec g(1,150)$ contains all graphs of class A, $\vec g(151,300)$ contains all graphs of class E, and, finally, $\vec g(301,450)$ contains all graphs of class F.

Half of the graphs in each original sequence $\vec g$ are randomly selected without repetition as training graphs, and are used to learn the graph embedding map $\map(\cdot)$.
The remaining graphs constitute the actual sequence on which the tests are run to identify the presence of change points. 
The adopted graph embedding technique is the dissimilarity representation \cite{pekalska2005dissimilarity}, which considers $d$ prototype graphs $\{r_1,\dots,r_d\}\in\mc G$ and maps each graph $g_i\in \mc G$ to a vector $x_i = (\graphdist(r_1,g_i),\dots,\graphdist(r_d,g_i))\in\R^d$ containing GAM evaluations of $g_i$ with respect to each prototype.
Here, $\graphdist(\cdot,\cdot)$ is built on the Euclidean kernel for real-valued attributes.
When the attribute set is of categorical data, the GAM is based on the delta-like kernel, which assigns 1 if the attributes are equal and 0 otherwise. 
The prototypes are selected among the training graphs according to the $d$-centers method \cite{riesen2007graph}, with $d>0$ being the number of centers (prototypes) and, hence, corresponds to the embedding dimension $d$.
The embedding dimension $d$ is a critical parameters that impacts on the quality of the embedding, affecting the sharpness of bounds of the form in \eqref{eq:fabs-sg-se}. Since the optimal value of embedding dimension $d$ depends on the specific data sets at hand, here we set $d$ to $3$, which turned out to be an effective choice in most of the cases. We leave a more focused study on the impact of the embedding dimension $d$ as future work.

The methodology presented in Section~\ref{sec:methodology} is evaluated on sequences with zero, one and multiple change points generated on the above data sets. To properly assess the performance, the experiments are run on the three proposed methods \mucpm\ (Section~\ref{sec:clt}), \ecpm\ (Section~\ref{sec:energy}), and \ediv\ (Section~\ref{sec:multi-change}) as instances of the proposed methodology. The significance level $\alpha$ of the test is set to $0.01$.

Four performance metrics are considered to evaluate the effectiveness of the methods.
The first two metrics are the true positive rate (TPR) and false positive rate (FPR), where ``positive'' refers to an actual change in the sequence. Here TPR ranges in $[0,1]$ and assesses the rate of detected changes over the total number of ground truth changes. Conversely, the FPR can be greater than 1, as it is the rate of changes identified beyond the ground truth ones over the total number of stationary sequences.
The third metric quantifies the mean distance of the detected change from the ground truth change-point.
In order to make it independent from the length $(b-a) + 1$ of sequence $\vec g(a, b)$, we consider the normalized discrepancy $|t^*-\hat t|/((b-a)+1)$ between change point $t^*$ and the estimated change point $\hat t$.
We call such a measure relative time-step error (RTE).
The last metric is the adjusted Rand index (ARI) \cite{morey1984measurement}, which is used for comparing two partitions of the same set. ARI ranges in the $[-1, 1]$ interval, with 1 corresponding to partitions in perfect agreement; an ARI equal to 0 is expected when the partitions are completely random; negative values denote partitions in disagreement.
To robustly estimate the aforementioned metrics and assess their variability, we repeated each experiment 100 times.

\subsection{Results}
\newcommand{\cisize}{\scriptsize}

%
\begin{table*}
\caption{Comparison of methods \mucpm,\ \ecpm,\ and \ediv\ on sequences of Delaunay graphs with single change point.
The problems are listed in increasing level of difficulty. Statistically better results are in bold. Not applicable measures are denoted with `---'.}
\label{tab:del}
\centering
\begin{tabular}{|cc|c>{\cisize}c|c>{\cisize}c|c>{\cisize}c|c>{\cisize}c|}
\hline

\hline
\multicolumn{2}{|c|}{\textbf{Experiment}} & \multicolumn{2}{c|}{\textbf{TPR}} & \multicolumn{2}{c|}{\textbf{FPR}} & \multicolumn{2}{c|}{\textbf{ARI}} & \multicolumn{2}{c|}{\textbf{RTE}} \\
    Seq. ID &  Method & mean        & 95\% c.i.    & mean        & 95\% c.i.       & mean        & 95\% c.i.       & mean        & 95\% c.i.   \\
\hline

\hline
Del\textsubscript{0,10} & \mucpm & 1.000   & [1.000, 1.000]        & --- & ---         & 1.000         & [1.000, 1.000]        & 0.000         & [0.000, 0.000]        \\
Del\textsubscript{0,12} & \mucpm & 1.000   & [1.000, 1.000]        & --- & ---         & 0.999         & [1.000, 1.000]        & 0.000         & [0.000, 0.000]        \\
Del\textsubscript{0,14} & \mucpm & 1.000   & [1.000, 1.000]        & --- & ---         & 0.998         & [1.000, 1.000]        & 0.000         & [0.000, 0.000]        \\
Del\textsubscript{0,16} & \mucpm & 0.010   & [0.000, 0.030]        & --- & ---         & 0.010         & [0.000, 0.000]        & 0.000         & [0.000, 0.000]        \\
Del\textsubscript{0,18} & \mucpm & 0.000   & [0.000, 0.000]        & --- & ---         & 0.000         & [0.000, 0.000]        &  --- & ---      \\
Del\textsubscript{0,20} & \mucpm & 0.000   & [0.000, 0.000]        & --- & ---         & 0.000         & [0.000, 0.000]        &  --- & ---      \\
\hline
Del\textsubscript{0,10} & \ecpm  & 1.000   & [1.000, 1.000]        & --- & ---         & 1.000         & [1.000, 1.000]        & 0.000         & [0.000, 0.000]        \\
Del\textsubscript{0,12} & \ecpm  & 1.000   & [1.000, 1.000]        & --- & ---         & 1.000         & [1.000, 1.000]        & 0.000         & [0.000, 0.000]        \\
Del\textsubscript{0,14} & \ecpm  & 1.000   & [1.000, 1.000]        & --- & ---         & 0.997         & [0.960, 1.000]        & 0.001         & [0.000, 0.010]        \\
Del\textsubscript{0,16} & \ecpm  & \textbf{0.890}   & [0.830, 0.950]        & --- & ---         & 0.797         & [0.000, 1.000]        & 0.030         & [0.000, 0.228]        \\
Del\textsubscript{0,18} & \ecpm  & \textbf{0.730}   & [0.640, 0.810]        & --- & ---         & 0.514         & [0.000, 1.000]        & 0.089         & [0.000, 0.310]        \\
Del\textsubscript{0,20} & \ecpm  & \textbf{0.270}   & [0.190, 0.360]        & --- & ---         & 0.193         & [0.000, 1.000]        & 0.084         & [0.000, 0.294]        \\
\hline
Del\textsubscript{0,10} & \ediv  & 1.000   & [1.000, 1.000]        & 0.010         & [0.010, 0.010]        & 0.942         & [0.939, 0.960]        & 0.017         & [0.010, 0.015]        \\
Del\textsubscript{0,12} & \ediv  & 1.000   & [1.000, 1.000]        & 0.005         & [0.005, 0.005]        & 0.956         & [0.960, 0.981]        & 0.011         & [0.005, 0.010]        \\
Del\textsubscript{0,14} & \ediv  & 1.000   & [1.000, 1.000]        & 0.000         & [0.000, 0.000]        & 0.960         & [0.960, 0.960]        & 0.010         & [0.010, 0.010]        \\
Del\textsubscript{0,16} & \ediv  & 0.660   & [0.570, 0.750]        & 0.010         & [0.010, 0.010]        & 0.588         & [0.000, 1.000]        & 0.030         & [0.000, 0.133]        \\
Del\textsubscript{0,18} & \ediv  & 0.320   & [0.230, 0.410]        & 0.005         & [0.005, 0.005]        & 0.237         & [0.000, 0.960]        & 0.075         & [0.008, 0.264]        \\
Del\textsubscript{0,20} & \ediv  & 0.040   & [0.010, 0.080]        & 0.005         & [0.005, 0.005]        & 0.026         & [0.000, 0.385]        & 0.122         & [0.020, 0.336]        \\
\hline
\end{tabular}
\end{table*}

\begin{figure}[t]
\centering
\includegraphics[width=.42\textwidth]{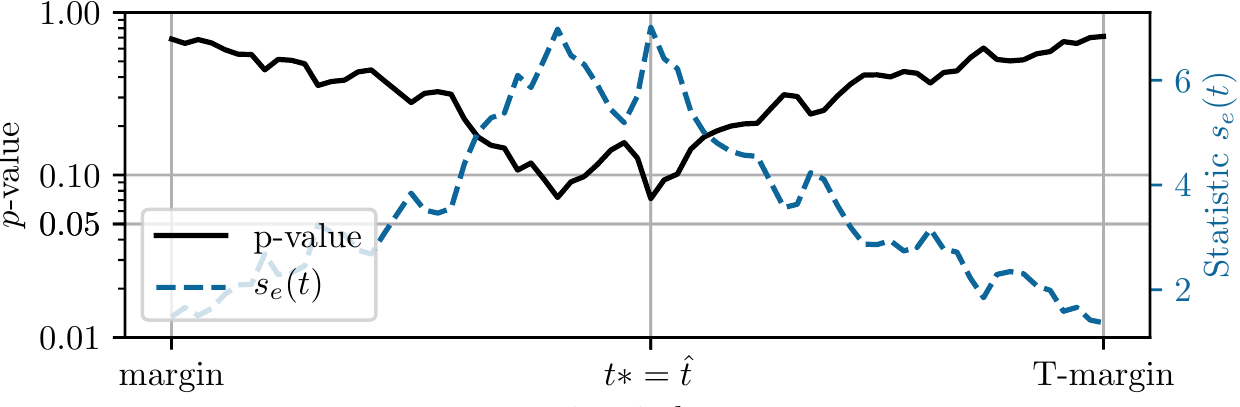}
\includegraphics[width=.42\textwidth]{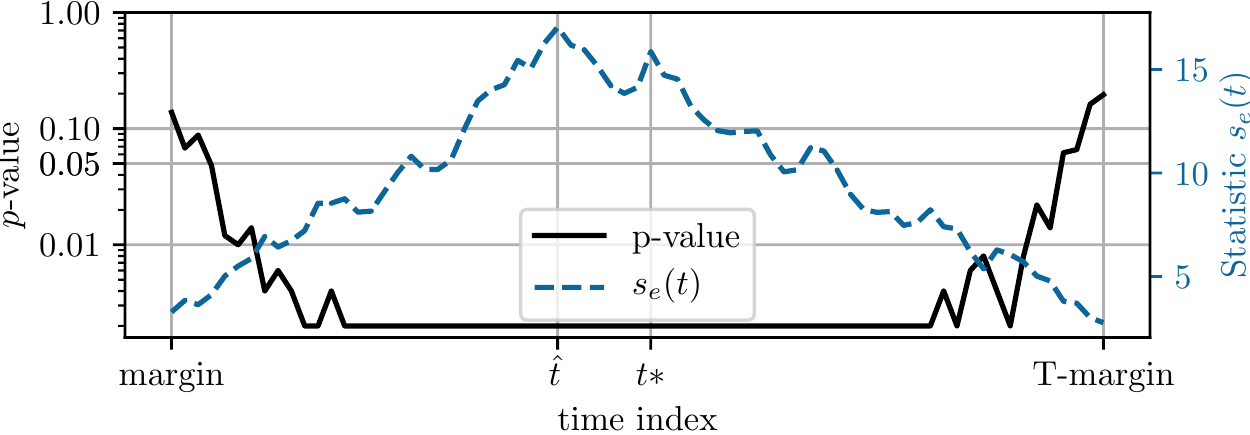}
\caption{The two sub-figures describe the behavior of \mucpm\ (top) and \ecpm\ (bottom) on the sequence Del\textsubscript{0,1}. The figure depicts the embedding statistic $s_e(t)$ and associated $\pval(t)$ at different time steps $t$ with a margin of $15$ time steps, $t=15,\dots,T-15$.}
\label{fig:cpm-example}
\end{figure}

\noindent\textbf{Delaunay graphs: }
We analyze the performance on single-change-point identification in sequences Del\textsubscript{0,$k$}, $k=8,10,12, \dots,20$.
From the results in Table~\ref{tab:del}, it can be observed that, up to class $k=14$, all CPMs perform very well on all metrics taken into account.
Starting from $k=16$, we note that \mucpm\ is not able to identify any change, whereas the other two methods are still able to detect most of them.
Despite the fact that \ecpm\ and \ediv\ are built on the same energy statistic \eqref{eq:energy-statistic}, when the problem becomes harder (i.e., $k\geq 16$), \ecpm\ performed significantly better than \ediv, at least in terms of TPR, where the 95\% confidence intervals do not intersect. We believe this result is related to the fact that \ecpm\ assumes 
that the sequence can contain no more than one change point,
whereas \ediv\ is more general and tries to identify multiple changes; 
however, in this case there is only one change point. Finally, we note that, in sequences where a change point is guaranteed to exist, it is impossible for both \mucpm\ and \ecpm\ to have a non-zero FPR, while this does not hold for \ediv.
Figure \ref{fig:cpm-example} illustrates how \mucpm\ and \ecpm\ operate on the Del\textsubscript{0,1}. \mucpm\ was able to identify the correct location, however the test did not reach the predefined confidence of $\alpha=0.01$, thus no change is actually detected. Conversely, \ecpm\ correctly identified the presence of the change ($p$-value smaller than $\alpha=0.01$) and estimated the change point slightly apart from the actual one.

\begin{table*}
\caption{Methods applied to graph sequences from the IAM graph database containing a single change point. As the Letter data set contains multiple classes, here we considered only the first two. Not applicable measurements are denoted with `---'.}
\label{tab:iam}
\centering
\begin{tabular}{|cc|c>{\cisize}c|c>{\cisize}c|c>{\cisize}c|c>{\cisize}c|}
\hline

\hline
\multicolumn{2}{|c|}{\textbf{Experiment}} & \multicolumn{2}{c|}{\textbf{TPR}} & \multicolumn{2}{c|}{\textbf{FPR}} & \multicolumn{2}{c|}{\textbf{ARI}} & \multicolumn{2}{c|}{\textbf{RTE}} \\
    Seq. ID &  Method & mean        & 95\% c.i.    & mean        & 95\% c.i.       & mean        & 95\% c.i.       & mean        & 95\% c.i.   \\
\hline

\hline
Let\textsubscript{A,E} & \mucpm  & 0.950   & [0.900, 0.990]        & --- & ---         & 0.946         & [0.000, 1.000]        & 0.001         & [0.000, 0.007]        \\
Let\textsubscript{A,E} & \ecpm   & 0.990   & [0.970, 1.000]        & --- & ---         & 0.987         & [0.973, 1.000]        & 0.001         & [0.000, 0.007]        \\
Let\textsubscript{A,E} & \ediv   & 1.000   & [1.000, 1.000]        & 0.000         & [0.000, 0.000]        & 0.974         & [0.947, 1.000]        & 0.007         & [0.000, 0.013]        \\
\hline
AIDS\textsubscript{0,1} & \mucpm & 0.770   & [0.690, 0.850]        & --- & ---         & 0.685         & [0.000, 1.000]        & 0.023         & [0.000, 0.184]        \\
AIDS\textsubscript{0,1} & \ecpm  & 1.000   & [1.000, 1.000]        & --- & ---         & 0.988         & [0.921, 1.000]        & 0.003         & [0.000, 0.017]        \\
AIDS\textsubscript{0,1} & \ediv  & 1.000   & [1.000, 1.000]        & 0.260         & [0.260, 0.260]        & 0.901         & [0.534, 0.996]        & 0.020         & [0.001, 0.103]        \\
\hline
Mut\textsubscript{0,1} & \mucpm   & 0.000   & [0.000, 0.000]        & --- & ---         & 0.000         & [0.000, 0.000]        &  --- & ---     \\
Mut\textsubscript{0,1} & \ecpm    & 1.000   & [1.000, 1.000]        & --- & ---         & 0.976         & [0.877, 1.000]        & 0.006         & [0.000, 0.032]        \\
Mut\textsubscript{0,1} & \ediv   & 0.990   & [0.970, 1.000]        & 1.050         & [1.050, 1.050]        & 0.305         & [-0.008, 0.972]       & 0.201         & [0.001, 0.473]        \\
\hline
\end{tabular}
\end{table*}
\begin{table}
\caption{Methods \mucpm,\ \ecpm,\ and \ediv\ applied to sequences formed by graphs belonging to a single class.}
\label{tab:no-change}
\centering
\begin{tabular}{|cc|c>{\cisize}c|}
\hline

\hline
\multicolumn{2}{|c|}{\textbf{Experiment}}  & \multicolumn{2}{c|}{\textbf{FPR}} \\
    Seq. ID &  Method & mean        & 95\% c.i   \\
\hline

\hline
Del\textsubscript{0} & \mucpm    & 0.000   & [0.000, 0.000]        \\
Del\textsubscript{0} & \ecpm     & 0.020   & [0.020, 0.020]        \\
Del\textsubscript{0} & \ediv    & 0.010   & [0.010, 0.010]        \\
\hline
Let\textsubscript{A} & \mucpm    & 0.000   & [0.000, 0.000]        \\
Let\textsubscript{A} & \ecpm     & 0.240   & [0.240, 0.240]        \\
Let\textsubscript{A} & \ediv    & 0.030   & [0.030, 0.030]        \\
\hline
AIDS\textsubscript{0} & \mucpm   & 0.000   & [0.000, 0.000]        \\
AIDS\textsubscript{0} & \ecpm    & 1.000   & [1.000, 1.000]        \\
AIDS\textsubscript{0} & \ediv   & 1.460   & [1.460, 1.460]        \\
\hline
Mut\textsubscript{0} & \mucpm   & 0.000   & [0.000, 0.000]        \\
Mut\textsubscript{0} & \ecpm    & 1.000   & [1.000, 1.000]        \\
Mut\textsubscript{0} & \ediv   & 1.290   & [1.290, 1.290]        \\
\hline
\end{tabular}
\end{table}

\noindent\textbf{IAM: }
Table~\ref{tab:iam} reports the results obtained on sequences composed by graphs from the IAM database.
Regarding the AIDS\textsubscript{1,0} data set, we note that \ediv\ presented a few false positives.
This may be caused by multiple factors, like dependence among graphs in the sequence or the presence of an actual change in stationarity within one of the two classes constituting the data set.
In fact, differently from the graphs of the Delaunay data set, which are i.i.d.\ by construction, the graphs in the IAM data sets are not guaranteed to be as such, as they represent objects from the real world (e.g., biological molecules).
We stress that each of the considered tests is configured to yield a FPR $\alpha=\prob(\text{reject }H_0|H_0\text{ is true})$. As such, the level $\alpha$ should be observable also in the experimental results, provided that the assumption of stationary is met. In fact, the FPR of \ediv\ on the Delaunay graphs shown in Table~\ref{tab:del}, where the sequences are guaranteed to be stationary, is consistent with the predefined rate $\alpha=0.01$.

To investigate further, we performed additional tests to assess whether the considered sequences of IAM graphs are stationary or not. In particular, we ran the CPMs on sequences composed by graphs belonging to a single class.
Table \ref{tab:no-change} shows the results, providing evidence that the sequences containing Delaunay graphs are stationary, while those composed by IAM graphs might not be stationary. In fact, while the \mucpm\ does not yield any false positive, the \ecpm\ and \ediv\ tests identified some changes. We also point out that \mucpm\ is designed to address changes in the distribution mean, suggesting that the intra-class changes in stationarity might affect moments beyond the first one.

In general, we note that the (injected) change point of each sequence has been identified by every method; the only, exception is Mut\textsubscript{0,1}, which constitutes a more challenging problem that is solved only by \ecpm.

\begin{table*}
\caption{\ediv\ applied to sequences presenting multiple change points. 
}
\label{tab:multi-change}
\centering
\begin{tabular}{|cc|c>{\cisize}c|c>{\cisize}c|c>{\cisize}c|c>{\cisize}c|}
\hline

\hline
\multicolumn{2}{|c|}{\textbf{Experiment}} & \multicolumn{2}{c|}{\textbf{TPR}} & \multicolumn{2}{c|}{\textbf{FPR}} & \multicolumn{2}{c|}{\textbf{ARI}} & \multicolumn{2}{c|}{\textbf{RTE}} \\
    Seq. ID &  Method & mean        & 95\% c.i.    & mean        & 95\% c.i.       & mean        & 95\% c.i.       & mean        & 95\% c.i.   \\
\hline

\hline
Del\textsubscript{0,6,8} & \ediv        & 1.000   & [1.000, 1.000]        & 0.007         & [0.007, 0.007]        & 0.917         & [0.162, 0.973]        & 0.009         & [0.007, 0.007]        \\
Del\textsubscript{12,14,16,18,20} & \ediv       & 0.708   & [0.662, 0.752]        & 0.000         & [0.000, 0.000]        & 0.344         & [-0.105, 0.982]       & 0.012         & [0.000, 0.072]        \\
\hline
Let\textsubscript{A,E,H} & \ediv        & 0.950   & [0.920, 0.980]        & 0.003         & [0.003, 0.003]        & 0.821         & [0.033, 1.000]        & 0.008         & [0.000, 0.038]        \\
Let\textsubscript{A,E,F,H,I} & \ediv    & 0.948   & [0.925, 0.968]        & 0.000         & [0.000, 0.000]        & 0.422         & [-0.091, 0.988]       & 0.005         & [0.000, 0.021]        \\
\hline
\end{tabular}
\end{table*}

\noindent\textbf{Multiple change points: }
We performed additional experiments on sequences where we injected multiple (i.e., more than one) change points.
To this end, we considered the Delaunay and Letter data sets, as they are composed of more than two classes of graphs.
Table \ref{tab:multi-change} shows that the TPR is statistically greater than 0.9 in three out of four cases, and the FPR is usually very small.
We note that the mean ARI is always close to zero, and it is significantly different from zero only in two settings.
Overall, we conclude that \ediv\ method was able to identify most of the change points, even in the most challenging sequence, Del\textsubscript{12,14,16,18,20}, which contains small changes in the distribution.

\begin{table*}
\caption{Methods \mucpm,\ \ecpm,\ and \ediv\ applied to sequences of graphs from the Kaggle database ``Detect seizures in intracranial EEG recordings''. Not applicable measurements are denoted with `---'.}
\label{tab:kaggle}
\centering
\begin{tabular}{|cc|c>{\cisize}c|c>{\cisize}c|c>{\cisize}c|c>{\cisize}c|}
\hline

\hline
\multicolumn{2}{|c|}{\textbf{Experiment}} & \multicolumn{2}{c|}{\textbf{TPR}} & \multicolumn{2}{c|}{\textbf{FPR}} & \multicolumn{2}{c|}{\textbf{ARI}} & \multicolumn{2}{c|}{\textbf{RTE}} \\
    Seq. ID &  Method & mean        & 95\% c.i.    & mean        & 95\% c.i.       & mean        & 95\% c.i.       & mean        & 95\% c.i.   \\
\hline

\hline
D1\textsubscript{0,1} & \mucpm & 0.190   & [0.120, 0.270]        & --- & --- & 0.160         & [0.000, 0.925]        & 0.039         & [0.015, 0.141]        \\
D1\textsubscript{0,1} & \ecpm  & 1.000   & [1.000, 1.000]        & --- & --- & 0.906         & [0.849, 0.951]        & 0.023         & [0.012, 0.037]        \\
D1\textsubscript{0,1} & \ediv  & 1.000   & [1.000, 1.000]        & 0.110         & [0.110, 0.110]        & 0.781         & [0.194, 0.931]        & 0.065         & [0.017, 0.278]        \\
\hline
D2\textsubscript{0,1} & \mucpm & 0.000   & [0.000, 0.000]        & --- & --- & 0.000         & [0.000, 0.000]        &  --- & ---      \\
D2\textsubscript{0,1} & \ecpm  & 1.000   & [1.000, 1.000]        & --- & --- & 0.049         & [-0.107, 0.163]       & 0.438         & [0.295, 0.733]        \\
D2\textsubscript{0,1} & \ediv  & 1.000   & [1.000, 1.000]        & 1.665         & [1.665, 1.665]        & 0.503         & [-0.011, 0.924]       & 0.172         & [0.006, 0.652]        \\
\hline
D3\textsubscript{0,1} & \mucpm & 0.620   & [0.520, 0.710]        & --- & --- & 0.430         & [0.000, 0.970]        & 0.036         & [0.003, 0.058]        \\
D3\textsubscript{0,1} & \ecpm  & 1.000   & [1.000, 1.000]        & --- & --- & 0.979         & [0.959, 1.000]        & 0.003         & [0.000, 0.006]        \\
D3\textsubscript{0,1} & \ediv  & 1.000   & [1.000, 1.000]        & 1.505         & [1.505, 1.505]        & 0.406         & [-0.051, 0.978]       & 0.195         & [0.002, 0.687]        \\
 \hline
D4\textsubscript{0,1} & \mucpm & 0.010   & [0.000, 0.030]        & --- & --- & 0.007         & [0.000, 0.000]        & 0.023         & [0.023, 0.023]        \\
D4\textsubscript{0,1} & \ecpm  & 1.000   & [1.000, 1.000]        & --- & --- & 0.893         & [0.420, 1.000]        & 0.019         & [0.000, 0.103]        \\
D4\textsubscript{0,1} & \ediv  & 1.000   & [1.000, 1.000]        & 5.510         & [5.510, 5.510]        & 0.439         & [0.035, 0.952]        & 0.268         & [0.001, 0.867]        \\
\hline
H1\textsubscript{0,1} & \mucpm & 0.690   & [0.600, 0.780]        & --- & --- & 0.415         & [0.000, 0.737]        & 0.110         & [0.066, 0.153]        \\
H1\textsubscript{0,1} & \ecpm  & 1.000   & [1.000, 1.000]        & --- & --- & 0.623         & [0.486, 0.779]        & 0.104         & [0.057, 0.149]        \\
H1\textsubscript{0,1} & \ediv  & 1.000   & [1.000, 1.000]        & 0.045         & [0.045, 0.045]        & 0.573         & [0.164, 0.761]        & 0.123         & [0.057, 0.291]        \\
\hline
H2\textsubscript{0,1} & \mucpm & 0.000   & [0.000, 0.000]        & --- & --- & 0.010         & [0.000, 0.000]        &  --- & ---      \\
H2\textsubscript{0,1} & \ecpm  & 1.000   & [1.000, 1.000]        & --- & --- & 0.713         & [0.498, 0.831]        & 0.029         & [0.023, 0.034]        \\
H2\textsubscript{0,1} & \ediv  & 1.000   & [1.000, 1.000]        & 2.770         & [2.770, 2.770]        & 0.244         & [-0.074, 0.831]       & 0.368         & [0.015, 0.853]        \\
\hline
H3\textsubscript{0,1} & \mucpm & 0.000   & [0.000, 0.000]        & --- & --- & 0.000         & [0.000, 0.000]        &  --- & ---      \\
H3\textsubscript{0,1} & \ecpm  & 1.000   & [1.000, 1.000]        & --- & --- & -0.050        & [-0.081, 0.058]       & 0.486         & [0.301, 0.536]        \\
H3\textsubscript{0,1} & \ediv  & 0.990   & [0.970, 1.000]        & 0.615         & [0.615, 0.615]        & 0.247         & [-0.073, 0.919]       & 0.439         & [0.020, 0.667]        \\
\hline
H4\textsubscript{0,1} & \mucpm & 0.000   & [0.000, 0.000]        & --- & --- & 0.030         & [0.000, 0.525]        &  --- & ---      \\
H4\textsubscript{0,1} & \ecpm  & 0.780   & [0.700, 0.860]        & --- & --- & 0.040         & [-0.000, 0.106]       & 0.404         & [0.333, 0.495]        \\
H4\textsubscript{0,1} & \ediv  & 0.870   & [0.800, 0.930]        & 0.450         & [0.450, 0.450]        & 0.130         & [0.000, 0.929]        & 0.382         & [0.154, 0.524]        \\
\hline
H5\textsubscript{0,1} & \mucpm & 0.000   & [0.000, 0.000]        & --- & --- & 0.000         & [0.000, 0.000]        &  --- & ---      \\
H5\textsubscript{0,1} & \ecpm  & 0.980   & [0.950, 1.000]        & --- & --- & 0.766         & [-0.010, 1.000]       & 0.048         & [0.000, 0.439]        \\
H5\textsubscript{0,1} & \ediv  & 0.850   & [0.780, 0.920]        & 1.250         & [1.250, 1.250]        & 0.300         & [-0.052, 0.979]       & 0.381         & [0.001, 0.818]        \\
\hline
H6\textsubscript{0,1} & \mucpm & 0.000   & [0.000, 0.000]        & --- & --- & 0.000         & [0.000, 0.000]        &  --- & ---      \\
H6\textsubscript{0,1} & \ecpm  & 1.000   & [1.000, 1.000]        & --- & --- & 0.954         & [0.818, 1.000]        & 0.005         & [0.000, 0.015]        \\
H6\textsubscript{0,1} & \ediv  & 0.970   & [0.930, 1.000]        & 0.315         & [0.315, 0.315]        & 0.668         & [-0.036, 0.994]       & 0.212         & [0.001, 0.557]        \\
\hline
H7\textsubscript{0,1} & \mucpm & 0.000   & [0.000, 0.000]        & --- & --- & 0.000         & [0.000, 0.000]        &  --- & ---      \\
H7\textsubscript{0,1} & \ecpm  & 1.000   & [1.000, 1.000]        & --- & --- & 0.967         & [0.903, 1.000]        & 0.004         & [0.000, 0.013]        \\
H7\textsubscript{0,1} & \ediv  & 1.000   & [1.000, 1.000]        & 9.085         & [9.085, 9.085]        & 0.202         & [-0.025, 0.706]       & 0.463         & [0.005, 0.904]        \\
\hline
H8\textsubscript{0,1} & \mucpm & 0.020   & [0.000, 0.050]        & --- & --- & 0.010         & [0.000, 0.000]        & 0.035         & [0.033, 0.038]        \\
H8\textsubscript{0,1} & \ecpm  & 1.000   & [1.000, 1.000]        & --- & --- & 0.792         & [0.706, 0.854]        & 0.033         & [0.025, 0.040]        \\
H8\textsubscript{0,1} & \ediv  & 1.000   & [1.000, 1.000]        & 5.160         & [5.160, 5.160]        & 0.303         & [-0.012, 0.634]       & 0.285         & [0.007, 0.699]        \\
\hline
\end{tabular}
\end{table*}

\noindent\textbf{Seizure detection: }
Here, we show results obtained on the Kaggle seizure detection problem, considering all available subjects.
By taking into account the results in Table~\ref{tab:kaggle}, it is possible to observe that \mucpm\ is often unable to identify changes. This might be due to the fact that the mean graph associated with an epileptic seizure is not sufficiently different from those representing the normal, baseline brain state. Stated in other terms, this result suggests that relevant changes affect higher-order moments of the distribution underlying the functional connectivity networks.
Conversely, the other two methods are able to recognize changes scoring good performance on all metrics, e.g., see D1\textsubscript{0,1} and H1\textsubscript{0,1}.
We note that, however, despite the TPR is generally high, the location of the estimated change point is not always accurate (quantified by an RTE different from zero), which is also confirmed by ARI values that, in some cases, are not statistically different from zero; see for example H3\textsubscript{0,1}.

\section{Conclusions}
\label{sec:conclusion}

We proposed a methodology to determine the point in time where a change in stationarity occurred in a finite sequence of attributed graphs. The methodology takes into account a very large class of graphs and consists in mapping graphs to an Euclidean domain, where the mathematics is more amenable and multivariate change point methods can be applied.
With Proposition~\ref{prop:fabs-sg-se}, we proved that the statistical inference attained in the embedding space can be used to draw conclusions concerning the original problem in the graph domain, and vice-versa.
Future research efforts will focus on weakening assumption (A2) regarding the support of the graph distribution, and, more importantly, on relaxing the constraint of using a metric distance between graphs.
The two proposed CPMs address the detection of changes in the mean of the graph distribution and more general changes in stationarity affecting higher-order moments of the distribution. We derived explicit bounds to make Proposition~\ref{prop:fabs-sg-se} applicable in practical statistical inference procedures.
We also proposed a method to detect multiple (i.e., more than one) change points based on the E-divisive approach.

Our contribution is mostly theoretical and, as such, of general applicability.
However, we demonstrated the practical usefulness of what proposed by considering both synthetic and real-world data.
Case studies include graphs representing biological molecules, images, and an application aimed at detecting the onset of epileptic seizures from iEEGs represented as functional connectivity networks.
Results of simulations showed that the proposed CPM tests are effective in relevant application scenarios.

\appendices

\section{}
\label{sec:proofs}

\subsection{Proof of Proposition \ref{prop:fabs-sg-se}}
To prove the proposition, we first need the following lemma.
\begin{lemma}
\label{lemma:bound-Psi-ell-u}
Consider a random variable $Y\sim P$ taking values in $\mc Y$ and two statistics $s_1(\cdot),s_2(\cdot):\mc Y\rightarrow \R_+$ with associated cumulative density functions $\Psi_1(\cdot)$, $\Psi_2(\cdot)$, respectively. If function $\ell:\R_+\rightarrow\R_+$, increasing and bijective, and $q$ is a constant in $(0,1]$, then for any $\gamma\geq0$
\begin{equation}
\label{eq:s1-s2-u}
\prob\left(s_1(Y)\leq u(s_2(Y))\right)\geq q
\\\Rightarrow 
\Psi_1(\gamma) \geq q\Psi_2(u^{-1}(\gamma)).
\end{equation}
\end{lemma}
\begin{proof}
For convenience, let us define the following variables:
\begin{equation*}
\begin{array}{rl}
A&   := \{y \in \mc Y :s_1(y)\leq u(s_2(y))\}\\
\pi(-|-)&   := \prob(s_1(Y)\leq u(\gamma) \ |\ s_2(Y) \leq\gamma, Y\in A )\\
\pi(-|+)&   := \prob(s_1(Y)\leq u(\gamma) \ |\ s_2(Y) >   \gamma, Y\in A).
\end{array}
\end{equation*}
\label{pf1}
By the law of total probability, and for any $\gamma\geq 0$,
\begin{multline*}
\prob(s_1(Y)\leq u(\gamma))
=\prob(S_1(Y)\leq u(\gamma)| Y \in A)     \prob(Y \in A) + \\
+\prob(s_1(Y)\leq u(\gamma)| Y \not\in A) \prob(Y \not\in A).
\end{multline*}
Lower-bounding the second addendum with zero and by hypothesis,
$
\prob(s_1(Y)\leq u(\gamma)) 
\ \geq\ \prob(s_1(Y)\leq u(\gamma)| Y \in A) \cdot q.
$

Notice that $\prob(s_1(Y)\leq u(\gamma)|s_2(Y)=\gamma,Y\in A)=1$ for all $\gamma\geq 0$, thanks to the event $Y\in A$; hence, we have $\pi(-|-)=1$.
Applying again the law of total probabilities,
\begin{multline*}
\prob(s_1(Y)\leq u(\gamma)| Y \in A) \\
=\pi(-|-) \Psi_2(\gamma) 
 +\pi(-|+) (1- \Psi_2(\gamma))
\geq 1\cdot \Psi_2(\gamma).
\end{multline*}
Combining with the above Part \ref{pf1}, we prove \eqref{eq:s1-s2-u}
\\$
\Psi_1(u(\gamma))=\prob(s_1(Y)\leq u(\gamma)) \geq q\cdot\Psi_2(\gamma).
$
\end{proof}

The proof follows from Lemma~\ref{lemma:bound-Psi-ell-u} applied to Equation~\eqref{eq:fabs-sg-se} expressed in the form

$\qquad\prob_{\vec g\sim Q_0^T}(s_g(t;\vec g)\leq s_e(t;\map(\vec g))+\lambda)\geq q,$

$\qquad\prob_{\vec g\sim Q_0^T}(s_e(t;\map(\vec g))\leq s_g(t;\vec g)+\lambda)\geq q.$

\noindent and where $Y=\vec g$, $P=Q_0^T$ and $s_1(\cdot),s_2(\cdot)$ are set alternatively to $s_g(t;\cdot)$ and $s_e(t;\map(\cdot))$.

\subsection{Proof of Proposition~\ref{prop:p-vals}}
Notice that $p$-value $p_g$ is
$$p_g=\prob_{\vec g\sim Q_0^T}(s_g(t; \vec g)>s_g(t;\vec g^*)|H_0)=1-\Psi_g(s_g(t;\vec g^*)).$$
Then, from Proposition~\ref{prop:fabs-sg-se}, Eq.~\ref{eq:cdf-bound}, it follows that:
\begin{equation}
\label{eq:bound-pval_g}
1 - q^{-1}\Psi_e(s_g(t;\vec g^*)+\lambda)
\leq p_g  \leq 
1 - q\Psi_e(s_g(t;\vec g^*)-\lambda),
\end{equation}
and, from hypothesis \eqref{eq:fabs-sg-se}, 
$\Psi_e(s_g(t;\vec g^*)-\lambda)\leq\Psi_e(s_e(t;\vec x^*)-2\,\lambda)=p_e''$ and  $\Psi_e(s_g(t;\vec g^*)+\lambda)\geq\Psi_e(s_e(t;\vec x^*)+2\,\lambda)=p_e'$ with probability $q$.

\subsection{Proof of Lemma \ref{lemma:fabs-sg-se-clt}}
\label{sec:proof-lemma:fabs-sg-se-clt}

The claim is proved by applying the Markov inequality \cite{roussas1997course} to $|s_e(t;\map(\vec g))-s_g(t;\vec g)|$. For any $\lambda>0$,
\begin{equation}
\begin{aligned}
\label{eq:markov-sg-se}
&\prob(|s_e(t;\map(\vec g))-s_g(t;\vec g)|\geq \lambda)
\\&\ \leq \lambda^{-1}\left(\expect\left[s_g(t;\vec g)\right]+\expect\left[s_e(t;\map(\vec g))\right]\right).
\end{aligned}
\end{equation}
Let us evaluate $\expect[s_e(t; \vec x)]$ first, where $\vec x=\map(\vec g)$. As the Mahalanobis distance \eqref{eq:mahal-distance} is bounded by the Euclidean one via the smallest%
\footnote{Any time $M$ is singular, $M$ can be made positive definite by reducing the embedding space dimension.}
eigenvalue $\lambda_d({M})$ of matrix $M$, then
\begin{equation*}
\begin{aligned}
s_e(t;\vec x) &\leq \tfrac{T}{\lambda_d({M})} \norm{\frmu_{\vec x(1,t-1)}-\frmu_{\vec x(t,T)}}{2}^2 \\
&\leq \tfrac{2\,T}{\lambda_d({M})}\left( \norm{\frmu_{\vec x(1,t-1)}-\frmu_{F}}{2}^2 + \norm{\frmu_{F}-\frmu_{\vec x(t,T)}}{2}^2\right).
\end{aligned}
\end{equation*}

Recall the notion of Fr\'echet variation $\frvar[F]$ of Section~\ref{sec:def-graph-mean}. We have $\frvar[F]=\expect_{x\sim F}\left[\norm{x-\mu_F}{2}^2\right]$. 
Moreover, $\expect\left[\norm{\frmu_{\vec x(t_1+1,t_2)}-\frmu_{F}}{2}^2\right] = (t_2-t_1)^{-1}\frvar[F]$,
which leads to
\begin{equation*}
\begin{aligned}
\expect_{\vec x\sim F^T}[s_e(t;\vec x)] &\leq 2 T \left(\tfrac{1}{t-1}+\tfrac{1}{T-t+1}\right)\lambda_d({M})^{-1} \frvar[F] \\
&= \tfrac{2 T^2}{(t-1)(T-t+1)}\lambda_d({M})^{-1} \frvar[F].
\end{aligned}
\end{equation*}
By (A2), $(\mc G,\graphdist)$ can be isometrically embedded into an Euclidean space. Hence, similarly, we prove $\expect[s_g(t;\vec g)]\leq\frac{2 T^2}{(t-1)(T-t+1)}\frvar[Q]$.

\subsection{Proof of Proposition \ref{prop:energy-metric-GAS}}
Theorem 4.23 in \cite{jain2016geometry} defines a condition under which there is an isometric mapping from a GAS to an Euclidean space. Such a condition limits the support of $\bigcup_{Q\in\mc D}Q$ [Assumption (A1)]. \cite{szekely2013energy} proved the statement for the energy distance for an Euclidean space.

\subsection{Proof of Lemma \ref{lemma:fabs-sg-se-energy}}
Being under the null hypothesis, from Eq.~\ref{eq:sample-energy-distance},
\begin{align}
\label{eq:expect-energy-distance}
\nonumber \expect_{\vec x}&[\mc E(\vec x(a,t-1),\vec x(t,b))] 
\\\nonumber &=\left\{ 2 - \tfrac{(t-a)(t-a-1)}{(t-a)^2} - \tfrac{(b-t+1)(b-t)}{(b-t+1)^2}\right\} \expect_{\vec x}\left[\norm{x-x'}{2}\right]
\\ &\leq \tfrac{b-a+1}{(t-a)(b-a+1)} \expect_{\vec x}\left[\norm{x-x'}{2}\right]
\end{align}
with $x,x'\sim F$ independent random vectors whose distribution $F$ derives from $Q$ through mapping $\map(\cdot)$.

The claim is proved by considering the Markov inequality as done for \eqref{eq:markov-sg-se}. In fact,
$\expect[s_e(t;\vec x)] = 1\cdot \expect[\norm{x - x'}{2}] \leq 2\,\expect[\norm{x-\frmu_F}{2}].
$
In the graph space, we obtain a similar bound: $\expect[s_g(t;\vec g)]\leq 2\,\expect[\graphdist(G,\frmu_Q)]$, for $G\sim Q$.

\subsection{Proof of Lemma \ref{lemma:fabs-sg-se-edivisive}}

The maximum in Eq.~\ref{eq:edivisive-se} can be upper-bounded by
 $\sum_{r=t}^b \mc E(\vec x(a,t-1),\vec x(t,r))$, so that,
employing Eq.~\ref{eq:expect-energy-distance}, 
$\expect_{\vec x}[s_e(t;\vec x)]\leq \expect_{\vec x}\left[\norm{x-x'}{2}\right] \sum_{r=t}^b\tfrac{r-a+1}{(t-a)(r-t+1)},
$ 
with
\begin{align*}
\sum_{r=t}^b\tfrac{r-a+1}{(t-a)(r-t+1)}& 
= \tfrac{b-t}{t-a}+\sum_{i=1}^{b-t+1}\tfrac{1}{i}
\leq \tfrac{b-t}{t-a}+1+\int_{1}^{b-t+1}\tfrac{1}{x}\,dx 
\\&\leq \tfrac{b-a}{t-a}+\log(b-a+1) =: C_{a,b}(t).
\end{align*}
Similarly, we also conclude that
$\expect_{\vec g}[s_g(t;\vec g)]\leq \expect_{\vec g}\left[\graphdist(G,G')\right] C_{a,b}(t),$
and the claim is proven by considering the Markov inequality as done for \eqref{eq:markov-sg-se}.

\section*{Acknowledgments}
This research is funded by the Swiss National Science Foundation project 200021\_172671: ``ALPSFORT: A Learning graPh-baSed framework FOr cybeR-physical sysTems''.

\bibliographystyle{IEEEtran}
\bibliography{sample}

\end{document}